\documentclass{article}

\usepackage[top=1in, bottom=1in, left=1in, right=1in]{geometry}

\usepackage[utf8]{inputenc} 
\usepackage[T1]{fontenc}    
\usepackage{hyperref}       
\usepackage{url}            
\usepackage{booktabs}       
\usepackage{amsfonts}       
\usepackage{nicefrac}       
\usepackage{microtype}      
\usepackage{amsmath,amssymb}
\usepackage{color,graphicx}
\usepackage{multicol}
\usepackage{multirow}
\usepackage{caption}
\usepackage{subcaption}
\usepackage[boxruled,linesnumbered]{algorithm2e}
\usepackage{algorithmic}
\usepackage{enumitem}
\usepackage{url}
\usepackage{tcolorbox}
\usepackage{fullwidth}
\usepackage{lmodern}

\newtheorem{theorem}{Theorem}[section]

\newtheorem{remark}[theorem]{Remark}

\newtheorem{lemma}[theorem]{Lemma}

\newtheorem{definition}[theorem]{Definition}

\newcommand{\eps}{\varepsilon}
\renewcommand{\epsilon}{\varepsilon}
\newcommand{\eat}[1]{}

\let\oldnl\nl
\newcommand{\nonl}{\renewcommand{\nl}{\let\nl\oldnl}}

\newcommand{\R}{\mathbb{R}}

\newcommand{\poly}{\operatorname{poly}}

\newcommand{\AR}{\mathsf{AR}}

\newcommand{\CG}{\mathbf{CGLSE}}
\newcommand{\CGk}{\mathsf{CGLSE_k}}

\newcommand{\Exp}{\mathsf{Exp}}
\newcommand{\SVD}{\mathsf{SVD}}
\newcommand{\Cov}{\mathsf{Cov}}

\usepackage{enumitem}

\newcommand{\calB}{\mathcal{B}}

\newcommand{\calG}{{\mathcal{G}}}

\newcommand{\calP}{{\mathcal{P}}}

\newcommand{\calX}{{\mathcal{X}}}

\allowdisplaybreaks

\title{\bf Coresets for Regressions with Panel Data}

\author{Lingxiao Huang \\ Huawei TCS Lab \and K. Sudhir \\ Yale University \and Nisheeth K. Vishnoi \\ Yale University}

\begin{document}

\maketitle

\begin{abstract}
	This paper introduces the problem of coresets for regression problems to panel data settings.
	We first define coresets for several variants of regression problems with panel data and then present efficient algorithms to construct coresets of size that depend polynomially on $1/\eps$ (where $\eps$ is the error parameter) and the number of regression parameters -- independent of the number of individuals in the panel data or the time units each individual is observed for.
	Our approach is based on the Feldman-Langberg framework in which a key step is to upper bound the “total sensitivity” that is roughly the sum of maximum influences of all individual-time pairs taken over all possible choices of regression parameters. 
	Empirically, we assess our approach with synthetic and real-world datasets; the coreset sizes constructed using our approach are much smaller than the full dataset and coresets indeed accelerate the running time of computing the regression objective.
\end{abstract}

\thispagestyle{empty}

{\tiny }\newpage

\thispagestyle{empty}

\tableofcontents
\newpage

\setcounter{page}{1}

\section{Introduction}
\label{sec:intro}
Panel data, represented as $X\in \R^{N\times T\times d}$ and $Y\in \R^{N\times T}$ where $N$ is the number of entities/individuals, $T$ is the number of time periods and $d$ is the number of features is widely used in statistics and applied machine learning. 
Such data track features of a cross-section of entities (e.g., customers)  longitudinally over time. 
Such data are widely preferred in supervised machine learning for more accurate prediction and unbiased inference of relationships between variables relative to cross-sectional data (where each entity is observed only once) \cite{hsiao2003analysis, baltagi2008econometric}. 

The most common method for inferring relationships between variables using observational data involves solving regression problems on panel data.
The main difference between regression on panel data when compared to cross-sectional data is that there may exist correlations within observations associated with entities over time periods.
Consequently, the regression problem for panel data is the following optimization problem over  regression variables $\beta \in \mathbb{R}^d$ and the covariance matrix $\Omega$ that is induced by the abovementioned correlations: 
$\min_{\beta\in \R^d, \Omega\in \R^{T\times T}} \sum_{i\in [N]}  (y_i - X_i \beta)^\top \Omega^{-1} (y_i - X_i \beta).$
Here $X_i\in \R^{T\times d}$ denotes the observation matrix of entity $i$ whose $t$-th row is $x_{it}$ and $\Omega$ is constrained to have largest eigenvalue at most 1 where $\Omega_{t t'}$ represents the correlation between time periods $t$ and $t'$. 
This regression model is motivated by the random effects model (Eq.~\eqref{eq:linear} and Appendix~\ref{sec:discussion}), common in the panel data literature~\cite{hoechle2007robust,griffiths1985theory,frees2004longitudinal}.
A common way to define the correlation between observations is an autocorrelation structure $\AR(q)$~\cite{haddad1998simple,lesage1999theory} whose covariance matrix $\Omega$ is induced by a vector $\rho\in \R^{q}$ (integer $q\geq 1$). 
This type of correlation results in the generalized least-squares estimator (GLSE), where the parameter space is $\calP = R^{d+q}$.

As the ability to track entities on various features in real-time has grown, panel datasets have grown massively in size. 
However, the size of these datasets limits the ability to apply standard learning algorithms due to space and time constraints.
Further, organizations owning data may want to share only a subset of data with others seeking to gain insights to mitigate privacy or intellectual property related risks.
Hence, a question arises: {\em can we construct a smaller subset of the panel data on which we can solve the regression problems with performance guarantees that are close enough to those obtained when working with the complete dataset?}

One approach to this problem is to appeal to the theory of  ``coresets.'' 
Coresets,  proposed in \cite{agarwal2004approximating}, are weighted subsets of the data that allow for fast approximate inference for a large dataset by solving the problem on the smaller coreset. 
Coresets have been developed for a variety of unsupervised and supervised learning problems; for a survey, see \cite{phillips2016coresets}. 
But, thus, far coresets have been developed only for $\ell_2$-regression cross-sectional data \cite{drineas2006sampling,li2013iterative,boutsidis2013near,cohen2015uniform,jubran2019fast}; no coresets have been developed for regressions on panel data -- an important limitation, given their widespread use and advantages.

Roughly, a coreset for cross-sectional data is  a weighted subset of observations associated with  entities that approximates the regression objective for every possible choice of regression parameters.
An idea, thus,  is to construct a coreset for each time period (cross-section) and output their union as a coreset for panel data.
However, this union contains at least $T$ observations which is undesirable since $T$ can be large.
Further, due to the covariance matrix $\Omega$, it is not obvious how to use this union to approximately compute regression objectives.
With panel data, one  needs to consider both how to sample entities, and within each entity how to sample observations across time.
Moreover, we also need to define how to compute regression objectives on such a coreset consisting of entity-time pairs. 

\noindent
\textbf{Our contributions.}
We initiate the study of coresets for versions of $\ell_2$-regression with panel data, including the ordinary least-squares estimator (OLSE; Definition~\ref{def:olse}), the generalized least-squares estimator (GLSE; Definition~\ref{def:glse}), and a clustering extension of GLSE (GLSE$_k$; Definition~\ref{def:glsek}) in which all entities are partitioned into $k$ clusters and each cluster shares the same regression parameters. 

Overall, we  formulate the definitions of coresets and propose efficient construction of $\eps$-coresets of sizes independent of $N$ and $T$. Our key contributions are:
\begin{enumerate}
\item We give a novel formulation of coresets for GLSE (Definition~\ref{def:coreset_glse}) and GLSE$_k$ (Definition~\ref{def:coreset_glsek}).
We represent the regression objective of GLSE as the sum of $NT$ sub-functions w.r.t. entity-time pairs, which enables us to define coresets similar to the case of cross-sectional data. 
For GLSE$_k$, the regression objective cannot be similarly decomposed due to the $\min$ operations in Definition~\ref{def:glsek}.
To deal with this issue, we define the regression objective on a coreset $S$ by including $\min$ operations.
\item Our coreset for OLSE is of size $O(\min\{\eps^{-2}d, d^2\})$ (Theorems~\ref{thm:olse} and~\ref{thm:OLS_acc}), based on a reduction to coreset for $\ell_2$-regression with cross-sectional data.
\item Our coreset for GLSE consists of at most $\tilde{O}(\eps^{-2}\max\{q^4d^2, q^3 d^3\})$ points  (Theorem~\ref{thm:coreset_glse}), independent of $N$ and $T$ as desired.
\item
Our coreset for GLSE$_k$ is of size $\poly(M,k,q,d,1/\eps)$ (Theorem~\ref{thm:coreset_glsek}) where $M$ upper bounds  the gap between the maximum individual regression objective of OLSE and the minimum one (Definition~\ref{def:bounded_dataset_main}).
We provide a matching lower bound $\Omega(N)$  (Theorem~\ref{thm:lower_main}) for $k,q,d\leq 2$, indicating that the coreset size should contain additional factors than $k,q,d,1/\eps$, justifying the $M$-bounded assumption.
\end{enumerate}

\noindent
Our coresets  for GLSE/GLSE$_k$ leverage the Feldman-Langberg (FL) framework~\cite{feldman2011unified} (Algorithms~\ref{alg:glse} and~\ref{alg:glsek}). 
The $\rho$ variables make the objective function of GLSE non-convex in contrast to the cross-sectional data setting where objective functions are convex.
Thus, bounding the ``sensitivity'' (Lemma~\ref{lm:sen_glse}) of each entity-time pair for GLSE, which is a key step in coreset construction using the FL framework, becomes significantly difficult.  
We handle this by upper-bounding the maximum effect of $\rho$, based on the observation that the gap between the regression objectives of GLSE and OLSE with respect to the same $\beta\in \R^d$ is always constant, which enables us to reduce the problem to the cross-sectional setting.
For GLSE$_k$, a key difficulty is that the clustering centers are \textit{subspaces} induced by regression vectors, instead of \textit{points} as in Gaussian mixture models or $k$-means. 
Hence, it is unclear how GLSE$_k$ can be reduced to projective clustering used in Gaussian mixture models; see~\cite{feldman2019coresets}. 
To bypass this, we consider observation vectors of an individual as one entity and design a two-staged framework in which the first stage selects a subset of individuals that captures the $\min$ operations in the objective function and the second stage applies our coreset construction for GLSE on each selected individuals.
As in the case of GLSE, bounding the ``sensitivity'' (Lemma~\ref{lm:sen_glsek}) of each entity for GLSE$_k$ is a key step at the first stage.
Towards this, we relate the total sensitivity of entities to a certain ``flexibility'' (Lemma~\ref{lm:sen_olsek}) of each individual regression objective which is, in turn, shown to be controlled by the $M$-bounded assumption (Definition~\ref{def:bounded_dataset_main}).

	We implement our GLSE coreset construction algorithm and  test
	it on  synthetic  and  real-world datasets while varying $\eps$. Our coresets perform well relative to uniform samples on multiple datasets with different generative distributions. Importanty, the relative performance is robust and better on datasets with outliers.
	The maximum empirical error of our coresets is always below the guaranteed $\eps$ unlike with uniform samples. Further, for comparable levels of empircal error, our coresets perform much better than uniform sampling in terms of sample size and coreset construction speed.  

	\subsection{Related work}
	\label{sec:related}
	
	With panel data, depending on different generative models, there exist several ways to define $\ell_2$-regression~\cite{hoechle2007robust,griffiths1985theory,frees2004longitudinal}, including the pooled model, the fixed effects model, the random effects model, and the random parameters model. 
	In this paper, we consider the random effects model (Equation~\eqref{eq:linear}) since the number of parameters is independent of $N$ and $T$ (see Section~\ref{sec:discussion} for more discussion).

	For cross-sectional data,  there is more than a decade of extensive work on coresets for regression; e.g., $\ell_2$-regression~\cite{drineas2006sampling,li2013iterative,boutsidis2013near,cohen2015uniform,jubran2019fast}, $\ell_1$-regression~\cite{clarkson2005subgradient,sohler2011subspace,clarkson2016fast}, generalized linear models~\cite{huggins2016coresets,molina2018core} and logistic regression~\cite{reddi2015communication,huggins2016coresets,munteanu2018coresets,Tolochinsky2018GenericCF}.
	The most relevant for our paper is $\ell_2$-regression (least-squares regression), which admits an $\eps$-coreset of size $O(d/\eps^2)$~\cite{boutsidis2013near} and an accurate coreset of size $O(d^2)$~\cite{jubran2019fast}.

	With cross-sectional data, coresets have been developed for a large family
	of problems in machine learning and statistics, including clustering~\cite{feldman2011unified,feldman2013turning,huang2020coresets}, mixture model~\cite{lucic2017training}, low rank approximation~\cite{cohen2017input}, kernel regression~\cite{zheng2017coresets} and logistic regression~\cite{munteanu2018coresets}.
	We refer interested readers to recent surveys~\cite{munteanu2018survey,feldman2020survey}. 
	It is interesting to investigate whether these results can be generalized to panel data.

	{There exist other variants of regression sketches beyond coreset, including weighted low rank approximation~\cite{Clarkson2017LowRankAA}, row sampling~\cite{cohen2015lp}, and subspace embedding~\cite{sohler2011subspace,meng2013low}. These methods mainly focus on the cross-sectional setting. It is interesting to investigate whether they can be adapted to the panel data setting that with an additional covariance matrix.}

\section{$\ell_2$-regression with panel data}
\label{sec:pre}

We consider the following generative model of $\ell_2$-regression: for $(i,t)\in [N]\times [T]$,
\begin{align}
\label{eq:linear}
\textstyle y_{it} = x_{it}^\top \beta_i + e_{it},
\end{align}
where $\beta_i\in \R^d$ and $e_{it}\in \R$ is the error term drawn from a normal distribution. 
Sometimes, we may include an additional entity or individual specified effect $\alpha_i\in \R$ so that the outcome can be represented by $y_{it} = x_{it}^\top \beta_i+ \alpha_i + e_{it}$. 
This is equivalent to Equation~\eqref{eq:linear} by appending an additional constant feature to each observation $x_{it}$.

\begin{remark}
	Sometimes, we may not observe individuals for all time periods, i.e., some observation vectors $x_{it}$ and their corresponding outcomes $y_{it}$ are missing.
	One way to handle this is to regard those missing individual-time pairs as $(x_{it},y_{it})=(0,0)$.
	Then, for any vector $\beta\in \R^d$, we have $y_{it}-x_{it}^\top \beta = 0$ for each missing individual-time pairs.
\end{remark}

\noindent
As in the case of cross-sectional data, we assume there is no correlation between individuals.
Using this assumption, the $\ell_2$-regression function can be represented as follows: for any regression parameters $\zeta\in \calP$ ($\calP$ is the parameter space),
$
\psi(\zeta) = \sum_{i\in [N]} \psi_i(\zeta),
$
where $\psi_i$ is the individual regression function.
Depending on whether there is correlation within individuals and whether $\beta_i$ is unique, there are several variants of $\psi_i$.
The simplest setting is when all $\beta_i$s are the same, say $\beta_i=\beta$, and there is no correlation within individuals.
This setting results in the ordinary least-squares estimator (OLSE); summarized in the following definition.

\begin{definition}[\bf{Ordinary least-squares estimator (OLSE)}]
	\label{def:olse}
	For an ordinary least-squares estimator (OLSE), the parameter space is $\R^d$ and for any $\beta\in \R^d$ the individual objective function is 
	\[
	\textstyle	\psi^{(O)}_i(\beta):=\sum_{t\in [T]} \psi^{(O)}_{it}(\beta) =\sum_{t\in [T]} (y_{it}-x_{it}^\top \beta)^2.
	\]
\end{definition}

\noindent
Consider the case when $\beta_i$ are the same but there may be correlations between time periods within individuals.
A common way to define the correlation is called autocorrelation $\AR(q)$~\cite{haddad1998simple,lesage1999theory}, in which there exists $\rho\in B^q$, where $q\geq 1$ is an integer and $B^q = \left\{x\in \R^q: \|x\|_2<1\right\}$, such that
\begin{align}
\label{eq:error}
\textstyle e_{it} = \sum_{a=1}^{\min\left\{t-1,q\right\}} \rho_a e_{i, t-a} + N(0,1).
\end{align}
This autocorrelation results in the generalized least-squares estimator (GLSE).

\begin{definition}[\bf{Generalized least-squares estimator (GLSE)}]
	\label{def:glse}
	For a generalized least-squares estimator (GLSE) with $\AR(q)$ (integer $q\geq 1$), the parameter space is $\R^d\times B^q$ and for any $\zeta=(\beta,\rho)\in \R^{d}\times B^q$ the individual objective function is $	\psi^{(G,q)}_i(\zeta):= \sum_{t\in [T]} \psi^{(G,q)}_{it}(\zeta)$ equal to
	\begin{eqnarray*}
		\textstyle			(1-\|\rho\|_2^2) (y_{i1}-x_{i1}^\top \beta)^2 
		+ \sum_{t=2}^{T}  \left((y_{it}-x_{it}^\top \beta)-\sum_{j=1}^{\min\left\{t-1,q\right\}} \rho_j (y_{i,t-j}-x_{i,t-j}^\top \beta)\right)^2.
	\end{eqnarray*}
\end{definition}

\noindent
The main difference from OLSE is that a sub-function $\psi^{(G,q)}_{it}$ is not only determined by a single observation $(x_{it},y_{it})$; instead, the objective of $\psi^{(G,q)}_{it}$ may be decided by up to $q+1$ contiguous observations $(x_{i,\max\left\{1,t-q\right\}},y_{i,\max\left\{1,t-q\right\}}),\ldots,(x_{it},y_{it})$.

Motivated by $k$-means clustering~\cite{tan2006cluster}, we  also consider a generalized setting of GLSE, called GLSE$_k$ ($k\geq 1$ is an integer), in which all individuals are partitioned into $k$ clusters and each cluster corresponds to the same regression parameters with respect to some GLSE.

\begin{definition}[\bf{GLSE$_k$: an extention of GLSE}]
	\label{def:glsek}
	Let  $k,q\geq 1$ be integers.
	For a GLSE$_k$, the parameter space is $\left(\R^{d}\times B^q\right)^k$ and for any $ \zeta=(\beta^{(1)},\ldots,\beta^{(k)}, \rho^{(1)},\ldots,\rho^{(k)})\in \left(\R^{d}\times B^q\right)^k$ the individual objective function is 
	$
	\psi^{(G,q,k)}_i(\zeta):=\min_{l\in [k]}\psi^{(G,q)}_i(\beta^{(l)},\rho^{(l)}). 
	$
\end{definition}

\noindent
GLSE$_k$ is a basic problem with applications in many real-world fields; as accounting for \textit{unobserved heterogeneity} in panel regressions is critical for unbiased estimates~\cite{arellano2002panel,halaby2004panel}.
Note that each individual selects regression parameters $(\beta^{(l)}, \rho^{(l)})$ ($l\in [k]$) that minimizes its individual regression objective for GLSE.
Note that GLSE$_1$ is exactly GLSE.
Also note that GLSE$_k$ can be regarded as a generalized version of clustered linear regression~\cite{ari2002clustered}, in which there is no correlation within individuals.

\section{Our coreset definitions for panel data}
\label{sec:coreset}

In this section, we show how to define coresets for regression on  panel data, including OLSE and GLSE. 
Due to the additional autocorrelation parameters, it is not straightforward to define coresets for GLSE as in the cross-sectional setting.
One way is to consider all observations of an individual as an indivisible group and select a collection of individuals as a coreset.
However, this construction results in a coreset of size depending on $T$, which violates the expectation that the coreset size should be independent of $N$ and $T$. 
To avoid a large coreset size, we introduce a generalized definition: coresets of a query space, which captures the coreset definition for OLSE and GLSE.

\begin{definition}[\bf{Query space~\cite{feldman2011unified,braverman2016new}}]
\label{def:query_space}
Let $\calX$ be a index set together with a weight function $u: \calX\rightarrow \R_{\geq 0}$.
Let $\calP$ be a set called queries, and $\psi_x:\calP\rightarrow \R_{\geq 0}$ be a given loss function w.r.t. $x\in \calX$.
The total cost of $\calX$ with respect to a query $\zeta\in \calP$ is
$
\psi(\zeta) := \sum_{x\in \calX} u(x)\cdot \psi_x(\zeta).
$
The tuple $(\calX,u,\calP,\psi)$ is called a query space.
Specifically, if $u(x)=1$ for all $x\in \calX$, we use $(\calX,\calP,\psi)$ for simplicity.
\end{definition}

\noindent
Intuitively, $\psi$ represents a linear combination of weighted functions indexed by $\calX$, and $\calP$ represents the ground set of $\psi$.
Due to the separability of $\psi$, we have the following coreset definition.

\begin{definition}[\bf{Coresets of a query space~\cite{feldman2011unified,braverman2016new}}]
\label{def:coreset_query}
Let $(\calX,u,\calP,\psi)$ be a query space and $\eps \in (0,1)$ be an error parameter.
An $\eps$-coreset of $(\calX,u,\calP,\psi)$ is a weighted set $S\subseteq \calX$ together with a weight function $w: S\rightarrow \R_{\geq 0}$ such that for any $\zeta\in \calP$,
$
\psi_S(\zeta):=\sum_{x\in S} w(x)\cdot \psi_x(\zeta) \in (1\pm \eps)\cdot \psi(\zeta).
$
\end{definition}

\noindent
Here, $\psi_S$ is a computation function over the coreset that is used to estimate the total cost of $\calX$.
By Definitions~\ref{def:olse} and~\ref{def:glse}, the regression objectives of OLSE and GLSE can be decomposed into $NT$ sub-functions.
Thus, we can apply the above definition to define coresets for OLSE and GLSE.
Note that OLSE is a special case of GLSE for $q=0$.
Thus, we only need to provide the coreset definition for GLSE.
We let $u=1$ and $\calP = \R^d\times B^q$. 
The index set of GLSE has the following form:
\begin{align*}
Z^{(G,q)} = \left\{z_{it}=\left(x_{i,\max\left\{1,t-q\right\}}, y_{i,\max\left\{1,t-q\right\}}, \ldots x_{it},y_{it}\right) : (i,t)\in [N]\times [T] \right\},
\end{align*}
where each element $z_{it}$ consists of at most $q+1$ observations.
Also, for every $z_{it}\in Z^{(G,q)}$ and $\zeta=(\beta,\rho)\in \calP$, the cost function $\psi_{it}$ is: if $t=1$,
$
\psi^{(G,q)}_{it}(\zeta)  = (1-\|\rho\|_2^2)\cdot (y_{i1}-x_{i1}^\top \beta)^2;
$
and if $t\neq 1$,
$
\psi^{(G,q)}_{it}(\zeta) =\left((y_{it}-x_{it}^\top \beta)-\sum_{j=1}^{\min\left\{t-1,q\right\}} \rho_j (y_{i,t-j}-x_{i,t-j}^\top \beta)\right)^2.
$
Thus, $(Z^{(G,q)},\calP,\psi^{(G,q)})$ is a query space of GLSE.\footnote{Here, we slightly abuse the notation by using $\psi^{(G,q)}_{it}(\zeta)$ instead of $\psi^{(G,q)}_{z_{it}}(\zeta)$.}
Then by Definition~\ref{def:coreset_query}, we have the following coreset definition for GLSE.

\begin{definition}[\bf{Coresets for GLSE}]
\label{def:coreset_glse}
Given a panel dataset $X\in \R^{N\times T\times d}$ and $Y\in \R^{N\times T}$, a constant $\eps \in (0,1)$, integer $q\geq 1$, and parameter space $\calP$, an $\eps$-coreset for GLSE is a weighted set $S\subseteq [N]\times [T]$ together with a weight function $w: S\rightarrow \R_{\geq 0}$ such that for any $\zeta=(\beta, \rho)\in \calP$,
\begin{align*}
\psi^{(G,q)}_S(\zeta):=\sum_{(i,t)\in S} w(i,t)\cdot \psi^{(G,q)}_{it}(\zeta) \in (1\pm \eps)\cdot \psi^{(G,q)}(\zeta).
\end{align*}
\end{definition}

\noindent
The weighted set $S$ is exactly an $\eps$-coreset of the query space $(Z^{(G,q)},\calP,\psi^{(G,q)})$.
Note that the number of points in this coreset $S$ is at most $(q+1)\cdot|S|$.
Specifically, for OLSE, the parameter space is $\R^d$ since $q=0$, and the corresponding index set is
$
Z^{(O)} = \left\{z_{it}=(x_{it},y_{it}) : (i,t)\in [N]\times [T] \right\}.
$
Consequently, the query space of OLSE is $(Z^{(O)},\R^d,\psi^{(O)})$.

\paragraph{Coresets for GLSE$_k$}
Due to the $\min$ operation in Definition~\ref{def:glsek}, the objective function $\psi^{(G,q,k)}$ can only be decomposed into sub-functions $\psi^{(G,q,k)}_i$ instead of individual-time pairs.
Then let $u=1$, $\calP^k = \left(\R^{d}\times B^q\right)^k$, and
$
Z^{(G,q,k)}=\left\{z_i=(x_{i1},y_{i1},\ldots, x_{iT},y_{iT}): i\in [N]\right\}.
$
We can regard $(Z^{(G,q,k)},\calP^k,\psi^{(G,q,k)})$ as a query space of GLSE$_k$.
By Definition~\ref{def:coreset_query}, an $\eps$-coreset of $(Z^{(G,q,k)},\calP^k,\psi^{(G,q,k)})$ is a subset $I_S\subseteq [N]$ together with a weight function $w': I_S\rightarrow \R_{\geq 0}$ such that for any $\zeta\in \calP^k$, 
\begin{eqnarray}
\label{eq:I_S}
\sum_{i\in I_S} w'(i)\cdot \psi^{(G,q,k)}_{i}(\zeta) \in (1\pm \eps)\cdot \psi^{(G,q,k)}(\zeta).
\end{eqnarray}
However, each $z_i\in Z^{(G,q,k)}$ consists of $T$ observations, and hence, the number of points in this coreset $S$ is $T\cdot |S|$.
To avoid the size dependence of $T$, we propose a new coreset definition for GLSE$_k$.
The intuition is to further select a subset of time periods to estimate $\psi^{(G,q,k)}_i$.

Given $S\subseteq [N]\times [T]$, we denote $I_S := \left\{i\in [N]: \exists t\in [T], s.t., (i,t)\in S \right\}$ as the collection of individuals that appear in $S$. 
Moreover, for each $i\in I_S$, we denote $J_{S,i}:=\left\{t\in [T]: (i,t)\in S\right\}$ to be the collection of observations for individual $i$ in $S$.

\begin{definition}[\bf{Coresets for GLSE$_k$}]
	\label{def:coreset_glsek}
	Given a panel dataset $X\in \R^{N\times T\times d}$ and $Y\in \R^{N\times T}$, constant $\eps \in (0,1)$, integer $k,q\geq 1$, and parameter space $\calP^k$, an $\eps$-coreset for GLSE$_k$ is a weighted set $S\subseteq [N]\times [T]$ together with a weight function $w: S\rightarrow \R_{\geq 0}$ such that for any $\zeta=(\beta^{(1)},\ldots,\beta^{(k)},\rho^{(1)},\ldots,\rho^{(k)})\in \calP^k$,
	\begin{align*}
	\psi^{(G,q,k)}_S(\zeta):=\sum_{i\in I_S} \min_{l\in [k]} \sum_{t\in J_{S,i}} w(i,t)\cdot \psi^{(G,q)}_{it}(\beta^{(l)},\rho^{(l)}) \in (1\pm \eps)\cdot \psi^{(G,q,k)}(\zeta).
	\end{align*}
\end{definition}

\noindent
The key is to incorporate $\min$ operations in the computation function $\psi^{(G,q,k)}_S$ over the coreset.
Similar to GLSE, the number of points in such a coreset $S$ is at most $(q+1)\cdot|S|$.

\section{Coresets for GLSE}
\label{sec:alg}

{
In this section, we show how to construct coresets for GLSE. 
We let the parameter space be $\calP_\lambda = \R^d\times B^q_{1-\lambda}$ for some constant $\lambda\in (0,1)$ where 
$B^q_{1-\lambda}=\left\{\rho\in \R^q: \|\rho\|_2^2\leq 1-\lambda \right\}$.
The assumption of the parameter space $B^q_{1-\lambda}$ for $\rho$ is based on the fact that $\|\rho\|_2^2< 1$ ($\lambda\rightarrow 0$) is a stationary condition for $\AR(q)$~\cite{lesage1999theory}.

\begin{theorem}[\bf{Coresets for GLSE}]
	\label{thm:coreset_glse}
	There exists a randomized algorithm that, for a given panel dataset $X\in \R^{N\times T\times d}$ and $Y\in \R^{N\times T}$, constants $\eps,\delta,\lambda \in (0,1)$ and integer $q\geq 1$, with probability at least $1-\delta$, constructs an $\eps$-coreset for GLSE of size 
	\[
	O\left(\eps^{-2} \lambda^{-1} q d\left(\max\left\{q^2d, qd^2\right\}\cdot \log \frac{d}{\lambda}+\log \frac{1}{\delta}\right) \right)
	\]
	and runs in time $O(NTq+NTd^2)$.
\end{theorem}

\noindent
\sloppy
Note that the coreset in the above theorem contains at most
$
(q+1)\cdot O\left(\eps^{-2}\lambda^{-1} qd\left(\max\left\{q^2d, qd^2\right\}\cdot \log \frac{d}{\lambda}+\log \frac{1}{\delta}\right) \right) 
$
points $(x_{it},y_{it})$, which is independent of both $N$ and $T$.
Also note that if both $\lambda$ and $\delta$ are away from 0, e.g., $\lambda=\delta=0.1$ the number of points in the coreset can be further simplified:
$
O\left(\eps^{-2} \max\left\{q^4 d^2, q^3 d^3\right\}\cdot \log d\right) = \poly(q,d,1/\eps).
$

\subsection{Algorithm for Theorem~\ref{thm:coreset_glse}}
\label{sec:algorithm_glse}

We summarize the algorithm of Theorem~\ref{thm:coreset_glse} in Algorithm~\ref{alg:glse}, which takes a panel dataset $(X,Y)$ as input and outputs a coreset $S$ of individual-time pairs.
The main idea is to use importance sampling (Lines 6-7) leveraging the Feldman-Langberg (FL) framework~\cite{feldman2011unified,braverman2016new}.
The key new step appears in Line 5, which computes a sensitivity function $s$ for GLSE that defines the sampling distribution.
Also note that the construction of $s$ is based on another function $s^{(O)}$ (Line 4), which is actually a sensitivity function for OLSE that has been studied in the literature~\cite{boutsidis2013near}.

	\begin{algorithm}[ht!]
		\caption{$\CG$: Coreset construction of GLSE}
		\label{alg:glse}
		\begin{algorithmic}[1]
			\REQUIRE {$X\in \R^{N\times T\times d}$, $Y\in \R^{N\times T}$, constant $\eps,\delta,\lambda \in (0,1)$, integer $q\geq 1$ and parameter space $\calP_\lambda$.}
			\ENSURE {a subset $S\subseteq [N]\times [T]$ together with a weight function $w:S\rightarrow \R_{\geq 0}$.} 
			\STATE 
			$M\leftarrow O\left(\eps^{-2} \lambda^{-1} q d\left(\max\left\{q^2d, qd^2\right\}\cdot \log \frac{d}{\lambda}+\log \frac{1}{\delta}\right) \right)$.
			\STATE Let $Z\in \R^{NT\times (d+1)}$ be whose $(iT-T+t)$-th row is $z_{it}=(x_{it},y_{it})\in \R^{d+1}$ for $(i,t)\in [N]\times [T]$.
			\STATE Compute $A\subseteq \R^{NT\times d'}$ whose columns form a unit basis of the column space of $Z$.
			\STATE For each $(i,t)\in [N]\times [T]$, $s^{(O)}(i,t)\leftarrow \|A_{iT-T+t}\|_2^2$.
			\STATE For each pair $(i,t)\in [N]\times [T]$, $s(i,t)\leftarrow
			\min\left\{1, 2\lambda^{-1} \left(s^{(O)}(i,t)+\sum_{j=1}^{\min\left\{t-1,q\right\}} s^{(O)}(i,t-j) \right)\right\}.
			$
			\STATE Pick a random sample $S\subseteq [N]\times [T]$ of $M$ pairs, where each $(i,t)\in S$ is selected with probability $\frac{s(i,t)}{\sum_{(i',t')\in [N]\times [T]}s(i',t')}$. 
			\STATE For each $(i,t)\in S$, $w(i,t)\leftarrow \frac{\sum_{(i',t')\in [N]\times [T]}s(i',t')}{M\cdot s(i,t)}$.
			\STATE Output $(S,w)$.
		\end{algorithmic}
	\end{algorithm}

\subsection{Useful notations and useful facts for Theorem~\ref{thm:coreset_glse}}
\label{sec:technical}
	
		Feldman and Langberg~\cite{feldman2011unified} show how to construct coresets by importance sampling and the coreset size has been improved by~\cite{braverman2016new}.

		\begin{theorem}[\bf{FL framework~\cite{feldman2011unified,braverman2016new}}]
			\label{thm:fl11}
			Let $\eps,\delta\in (0,1)$.
			Let $\dim$ be an upper bound of the pseudo-dimension.
			Suppose $s:[N]\times [T]\rightarrow \R_{\geq 0}$ is a sensitivity function satisfying that for any $(i,t)\in [N]\times [T]$,
			$
			s(i,t) \geq \sup_{\zeta\in \calP_\lambda} \frac{\psi^{(G,q)}_{it}(\zeta)}{\psi^{(G,q)}(\zeta)},
			$
			and $\calG := \sum_{(i,t)\in [N]\times [T]} s(i,t)$.
			Let $S\subseteq \calX$ be constructed by taking 
			\[
			O\left(\eps^{-2} \calG (\dim \cdot \log \calG +\log(1/\delta))\right)
			\]
			samples, where each sample $x\in \calX$ is selected with probability $\frac{s(x)}{\calG}$ and has weight $w(x):= \frac{\calG}{|S|\cdot s(x)}$.
			Then, with probability at least $1-\delta$, $S$ is an $\eps$-coreset for GLSE.
		\end{theorem}
		
		\noindent
		Here, the sensitivity function $s$ measures the maximum influence for each $x_{it}\in X$. 
		Note that the above is an importance sampling framework that takes samples from a distribution proportional to sensitivities.
		The sample complexity is controlled by the total sensitivity $\calG$ and the pseudo-dimension $\dim$.
		Hence, to apply the FL framework, we  need to upper bound the pseudo-dimension and construct a sensitivity function.

	\subsection{Proof of Theorem~\ref{thm:coreset_glse}}
	\label{sec:proof_glse}
	
	Algorithm~\ref{alg:glse} applies the FL framework (Feldman and Langberg~\cite{feldman2011unified}) that constructs coresets by importance sampling and the coreset size has been improved by~\cite{braverman2016new}.
	The key is to verify the “pseudo-dimension” (Lemma~\ref{lm:dim_glse}) and “sensitivities” (Lemma~\ref{lm:sen_glse}) separately; summarized as follows.

	\paragraph{Upper bounding the pseudo-dimension.}
	We have the following lemma that upper bounds the pseudo-dimension of $(Z^{(G,q)},\calP_\lambda,\psi^{(G,q)})$.

	\begin{lemma}[\bf{Pseudo-dimension of GLSE}]
		\label{lm:dim_glse}
		\sloppy
		The pseudo-dimension of any query space $(Z^{(G,q)},u,\calP_\lambda,\psi^{(G,q)})$ over weight functions $u: [N]\times [T]\rightarrow \R_{\geq 0}$ is at most $	O\left((q+d)qd \right)$.
	\end{lemma}
	
	\noindent
	The proof can be found in Section~\ref{sec:dim}.
	The main idea is to apply the prior results~\cite{anthony2009neural,vidyasagar2002theory} which shows that the pseudo-dimension is polynomially dependent on the number of regression parameters ($q+d$ for GLSE) and the number of operations of individual regression objectives ($O(qd)$ for GLSE).
	Consequently, we obtain the bound $O\left((q+d)qd \right)$ in Lemma~\ref{lm:dim_glse}.

	\paragraph{Constructing a sensitivity function.}
	Next, we show that the function $s$ constructed in Line 5 of Algorithm~\ref{alg:glse} is indeed a sensitivity function of GLSE that measures the maximum influence for each $x_{it}\in X$; summarized by the following lemma.

	\begin{lemma}[\bf{Total sensitivity of GLSE}]
		\label{lm:sen_glse}
		Function $s:[N]\times [T]\rightarrow \R_{\geq 0}$ of Algorithm~\ref{alg:glse} satisfies that for any $(i,t)\in [N]\times [T]$, $s(i,t) \geq \sup_{\zeta\in \calP} \frac{\psi^{(G,q)}_{it}(\zeta)}{\psi^{(G,q)}(\zeta)}$ and $\calG:=\sum_{(i,t)\in [N]\times [T]} s(i,t) = O(\lambda^{-1}qd)$. 
		Moreover, the construction time of function $s$ is $O(NTq+NTd^2)$.
	\end{lemma}
	
	\noindent
	Intuitively, if the sensitivity $s(i,t)$ is large, e.g., close to 1, $\psi^{(G,q)}_{it}$ must contribute significantly to the objective with respect to some parameter $\zeta\in \calP_\lambda$.
	The sampling ensures that we are likely to include such pair $(i,t)$ in the coreset for estimating $\psi(\zeta)$.
	Due to the fact that the objective function of GLSE is non-convex which is different from OLSE, bounding the sensitivity of each individual-time pair for GLSE becomes significantly difficult. 
	To handle this difficulty, we develop a reduction of sensitivities from GLSE to OLSE (Line 5 of Algorithm~\ref{alg:glse}), based on the relations between $\psi^{(G,q)}$ and $\psi^{(O)}$, i.e., for any $\zeta=(\beta,\rho)\in \calP_\lambda$ we prove that
	$
	\psi^{(G,q)}_i(\zeta) \geq \lambda \cdot \psi^{(O)}_i(\beta) \text{ and } 
	\psi^{(G,q)}_{it}(\zeta) \leq 2\cdot\left(\psi^{(O)}_{it} (\beta) +\sum_{j=1}^{\min\left\{t-1,q\right\}} \psi^{(O)}_{i,t-j}(\beta)\right).
	$
	The first inequality follows from the fact that the smallest eigenvalue of $\Omega_\rho^{-1}$ (the inverse covariance matrix induced by $\rho$) is at least $\lambda$.
	The intuition of the second inequality is from the form of function $\psi^{(G,q)}_{it}$, which relates to $\min\left\{t, q+1\right\}$ individual-time pairs, say $(x_{i,\min\left\{1,t-q\right\}}, y_{i,\min\left\{1, t-q\right\}}), \ldots, (x_{it},y_{it})$.
	Combining these two inequalities, we obtain a relation between the sensitivity function $s$ for GLSE and the sensitivity function $s^{(O)}$ for OLSE, 
	based on the following observation: for any $\zeta=(\beta,\rho)\in \calP_\lambda$,
	\begin{align*}
	\frac{\psi^{(G,q)}_{it}(\zeta)}{\psi^{(G,q)}(\zeta)} 
	&\leq &&\frac{2\cdot\left(\psi^{(O)}_{it} (\beta) + \sum_{j=1}^{\min\left\{t-1,q\right\}} \psi^{(O)}_{i,t-j}(\beta)\right)}{\lambda \cdot \psi^{(O)}(\beta)} \\
	&\leq && 2\lambda^{-1}\cdot \left(s^{(O)}(i,t)+{\textstyle\sum}_{j=1}^{\min\left\{t-1,q\right\}} s^{(O)}(i,t-j) \right) \\
	&= && s(i,t).
	\end{align*}
	which leads to the construction of $s$ in Line 5 of Algorithm~\ref{alg:glse}.
	Then it suffices to construct $s^{(O)}$ (Lines 2-4 of Algorithm~\ref{alg:glse}), which reduces to the cross-sectional data setting and has total sensitivity at most $d+1$ (Lemma~\ref{lm:sen_olse}).
	Consequently, we conclude that the total sensitivity $\calG$ of GLSE is $O(\lambda^{-1}qd)$ by the definition of $s$.

	Now we are ready to prove Theorem~\ref{thm:coreset_glse}.

	\begin{proof}[Proof of Theorem~\ref{thm:coreset_glse}]
		By Lemma~\ref{lm:sen_glse}, the total sensitivity $\calG$ is $O(\lambda^{-1} q d)$.
		By Lemma~\ref{lm:dim_glse}, we let $\dim = O\left((q+d)qd \right)$.
		Pluging the values of $\calG$ and $\dim$ in the FL framework~\cite{feldman2011unified,braverman2016new}, we prove for the coreset size.
		For the running time, it costs $O(NTq+NTd^2)$ time to compute the sensitivity function $s$ by Lemma~\ref{lm:sen_glse}, and $O(NTd)$ time to construct an $\eps$-coreset.
		This completes the proof.
	\end{proof}
	
\subsection{Proof of Lemma~\ref{lm:dim_glse}: Upper bounding the pseudo-dimension}
\label{sec:dim}

Our proof idea is similar to that in~\cite{lucic2017training}.
For preparation, we need the following lemma which is proposed to bound the pseudo-dimension of feed-forward neural networks.

\begin{lemma}[\bf{Restatement of Theorem 8.14 of~\cite{anthony2009neural}}]
	\label{lm:dim_bound}
	\sloppy
	Let $(\calX,u,\calP,f)$ be a given query space where $f_x(\zeta)\in \left\{0,1\right\}$ for any $x\in \calX$ and $\zeta\in \calP$, and $\calP\subseteq \R^{m}$.
	Suppose that $f$ can be computed by an algorithm that takes as input the pair $(x,\zeta)\in \calX\times \calP$ and returns $f_x(\zeta)$ after no more than $l$ of the following operations: 
	\begin{itemize}
		\item the arithmetic operations $+,-,\times$, and $/$ on real numbers.
		\item jumps conditioned on $>,\geq,<,\leq,=$, and $\neq$ comparisons of real numbers, and
		\item output 0,1.
	\end{itemize}
	Then the pseudo-dimension of $(\calX,u,\calP,f)$ is at most $O(ml)$.
\end{lemma}

\noindent
Note that the above lemma requires that the range of functions $f_x$ is $[0,1]$.
We have the following lemma which can help extend this range to $\R$.

\begin{lemma}[\bf{Restatement of Lemma 4.1 of~\cite{vidyasagar2002theory}}]
	\label{lm:dim_range}
	Let $(\calX,u,\calP,f)$ be a given query space.
	Let $g_x:\calP\times \R\rightarrow \left\{0,1\right\}$ be the indicator function satisfying that for any $x\in \calX$, $\zeta\in \calP$ and $r\in \R$,
	\[
	g_x(\zeta,r) = I\left[u(x)\cdot f(x,\zeta)\geq r\right].
	\]
	Then the pseudo-dimension of $(\calX,u,\calP,f)$ is precisely the pseudo-dimension of the query space $(\calX,u,\calP\times \R,g_f)$.
\end{lemma}

\noindent
Now we are ready to prove Lemma~\ref{lm:dim_glse}.

\begin{proof}[Proof of Lemma~\ref{lm:dim_glse}]
	Fix a weight function $u: [N]\times [T]\rightarrow \R_{\geq 0}$.
	For every $(i,t)\in [N]\times [T]$, let $g_{it}: \calP_\lambda\times \R_{\geq 0}\rightarrow \left\{0,1\right\}$ be the indicator function satisfying that for any $\zeta\in \calP_\lambda$ and $r\in \R_{\geq 0}$,
	\[
	g_{it}(\zeta,r) := I\left[u(i,t)\cdot \psi^{(G,q)}_{it}(\zeta)\geq r\right].
	\]
	We consider the query space $(Z^{(G,q)},u,\calP_\lambda\times \R_{\geq 0},g)$.
	By the definition of $\calP_\lambda$, the dimension of $\calP_\lambda\times \R_{\geq 0}$ is $m=q+1+d$.
	By the definition of $\psi^{(G,q)}_{it}$, $g_{it}$ can be calculated using $l=O(qd)$ operations, including $O(qd)$ arithmetic operations and a jump.
	Pluging the values of $m$ and $l$ in Lemma~\ref{lm:dim_bound}, the pseudo-dimension of $(Z^{(G,q)},u,\calP_\lambda\times \R_{\geq 0},g)$ is $O\left((q+d)qd\right)$.
	Then by Lemma~\ref{lm:dim_range}, we complete the proof.
\end{proof}

\subsection{Proof of Lemma~\ref{lm:sen_glse}: Bounding the total sensitivity}
\label{sec:sen}

We prove Lemma~\ref{lm:sen_glse} by relating sensitivities between GLSE and OLSE.
For preparation, we give the following lemma that upper bounds the total sensitivity of OLSE.
Given two integers $a,b\geq 1$, denote $T(a,b)$ to be the computation time of a column basis of a matrix in $\R^{a\times b}$.
For instance, a column basis of a matrix in $\R^{a\times b}$ can be obtained by computing its SVD decomposition, which costs $O(\min\left\{a^2 b, ab^2\right\})$ time by~\cite{cline2006computation}.

\begin{lemma}[\bf{Total sensitivity of OLSE}]
	\label{lm:sen_olse}
	Function $s^{(O)}:[N]\times [T]\rightarrow \R_{\geq 0}$ of Algorithm~\ref{alg:glse} satisfies that for any $(i,t)\in [N]\times [T]$,
	\begin{align}
	\label{ineq:sen_olse}
	s^{(O)}(i,t) \geq \sup_{\beta\in \R^d} \frac{\psi^{(O)}_{it}(\beta)}{\psi^{(O)}(\beta)},
	\end{align}
	and $\calG^{(O)} := \sum_{(i,t)\in [N]\times [T]} s^{(O)}(i,t)$ satisfying $\calG^{(O)} \leq d+1$.
	Moreover, the construction time of function $s^{(O)}$ is  $T(NT,d+1)+O(NTd)$.
\end{lemma}

\begin{proof}
	The proof idea comes from~\cite{varadarajan2012sensitivity}.
	By Line 3 of Algorithm~\ref{alg:glse}, $A\subseteq \R^{NT\times d'}$ is a matrix whose columns form a unit basis of the column space of $Z$.
	We have $d'\leq d+1$ and hence $\|A\|_2^2 = d'\leq d+1$.
	Moreover, for any $(i,t)\in [N]\times [T]$ and $\beta'\in \R^{d'}$, we have
	\[
	\|\beta'\|_2^2 \leq \|A \beta'\|_2^2,
	\]
	Then by Cauchy-Schwarz and orthonormality of $A$, we have that for any $(i,t)\in [N]\times [T]$ and $\beta'\in \R^{d+1}$,
	\begin{align}
	\label{ineq:sen}
	|z_{it}^\top \beta'|^2 \leq \|A_{iT-T+t}\|_2^2\cdot \|Z \beta'\|_2^2,
	\end{align}
	where $A_{iT-T+t}$ is the $(iT-T+t)$-th row of $A$.

	For each $(i,t)\in [N]\times [T]$, we let $s^{(O)}(i,t):=\|A_{iT-T+t}\|_2^2$.
	Then $\calG^{(O)} = \|A\|_2^2 =d'\leq d+1$.
	Note that constructing $A$ costs $T(NT,d+1)$ time and computing all $\|A_{iT-T+t}\|_2^2$ costs $O(NTd)$ time.

	Thus, it remains to verify that $s^{(O)}(i,t)$ satisfies Inequality~\eqref{ineq:sen_olse}.
	For any $(i,t)\in [N]\times [T]$ and $\beta\in \R^d$, letting $\beta'=(\beta,-1)$, we have
	\begin{eqnarray*}
		\begin{split}
			\psi^{(O)}_{it}(\beta) &= && |z_{it}^\top \beta'|^2 && (\text{Defn. of $\psi^{(O)}_{it}$}) \\
			& \leq && \|A_{iT-T+t}\|_2^2\cdot \|Z \beta'\|_2^2 && (\text{Ineq.~\eqref{ineq:sen}}) \\
			& = && \|A_{iT-T+t}\|_2^2\cdot \psi^{(O)}(\beta). && (\text{Defn. of $\psi^{(O)}$})
		\end{split}
	\end{eqnarray*}
	This completes the proof.
\end{proof}

\noindent 
Now we are ready to prove Lemma~\ref{lm:sen_glse}.

\begin{proof}[Proof of Lemma~\ref{lm:sen_glse}]
	For any $(i,t)\in [N]\times [T]$, recall that $s(i,t)$ is defined by
	\[
	s(i,t):=\min \left\{1, 2\lambda^{-1}\cdot \left(s^{(O)}(i,t)+{\textstyle\sum}_{j=1}^{\min\left\{t-1,q\right\}} s^{(O)}(i,t-j) \right)\right\}.
	\]
	We have that
	\begin{eqnarray*}
		\begin{split}
			\sum_{(i,t)\in [N]\times [T]} s(i,t) &\leq && \sum_{(i,t)\in [N]\times [q]} 2\lambda^{-1} \times \left(s^{(O)}(i,t)+{\textstyle\sum}_{j=1}^{\min\left\{t-1,q\right\}}s^{(O)}(i,t-j) \right) &  (\text{by definition})\\
			& \leq && 2\lambda^{-1}\cdot {\textstyle\sum}_{(i,t)\in [N]\times [T]} (1+q)\cdot s^{(O)}(i,t)  && \\
			& \leq && 2\lambda^{-1} (q+1)(d+1). & (\text{Lemma~\ref{lm:sen_olse}})
		\end{split}
	\end{eqnarray*}
	Hence, the total sensitivity $\calG = O(\lambda^{-1}qd)$.
	By Lemma~\ref{lm:sen_olse}, it costs $T(NT,d+1)+O(NTd)$ time to construct $s^{(O)}$.
	We also know that it costs $O(NTq)$ time to compute function $s$.
	Since $T(NT,d+1) = O(NTd^2)$, this completes the proof for the running time.

	Thus, it remains to verify that $s(i,t)$ satisfies that
	\[
	s(i,t) \geq \sup_{\zeta\in \calP} \frac{\psi^{(G,q)}_{it}(\zeta)}{\psi^{(G,q)}(\zeta)}.
	\]
	Since $\sup_{\beta\in \R^d} \frac{\psi^{(O)}_{it}(\beta)}{\psi^{(O)}(\beta)}\leq 1$ always holds, we only need to consider the case that
	\[
	s(i,t)=2\lambda^{-1}\cdot \left(s^{(O)}(i,t)+{\textstyle\sum}_{j=1}^{\min\left\{t-1,q\right\}} s^{(O)}(i,t-j) \right).
	\]
	We first show that for any $\zeta=(\beta,\rho)\in \calP_\lambda$, 
	\begin{align}
	\label{ineq:relation}
	\psi^{(G,q)}(\zeta) \geq \lambda \cdot \psi^{(O)}(\beta).
	\end{align}
	Given an autocorrelation vector $\rho\in \R^q$, the induced covariance matrix $\Omega_\rho$ satisfies that $\Omega_\rho^{-1}=P_\rho^\top P_\rho$ where
	\begin{eqnarray}
		\label{eq:cov_glse}
		\begin{split}
			& P_\rho = 
			\begin{bmatrix}
				\sqrt{1-\|\rho\|_2^2} & 0 & 0 & \ldots & \ldots& \ldots & 0 \\
				-\rho_1 & 1 &  0 & \ldots & \ldots& \ldots & 0 \\
				-\rho_2 & -\rho_1 & 1 & \ldots & \ldots& \ldots & 0 \\
				\ldots & \ldots & \ldots & \ldots  & \ldots& \ldots & \ldots \\
				0 & 0 & 0 & -\rho_q & \ldots& -\rho_1 & 1
			\end{bmatrix}.
		\end{split}
	\end{eqnarray}
	Then by Equation~\eqref{eq:cov_glse}, the smallest eigenvalue of $P_\rho$ satisfies that 
	\begin{eqnarray}
	\label{ineq:eigenvalue}
	\begin{split}
	\lambda_{\min} &=&& \sqrt{1-\|\rho\|_2^2} && (\text{Defn. of $P_\rho$})\\
	&\geq&& \sqrt{\lambda}. && (\rho\in B^q_{1-\lambda})
	\end{split}
	\end{eqnarray}
	Also we have
	\begin{eqnarray*}
		\begin{split}
			\psi^{(G,q)}(\zeta) & = && \sum_{i\in [N]} (y_i-X_i \beta)^\top \Omega_\rho^{-1} (y_i-X_i \beta) & (\text{Program (GLSE)}) \\
			& = && \sum_{i\in [N]} \|P_\rho (y_i-X_i \beta)\|_2^2 &(P_\rho^\top P_\rho = \Omega_\rho^{-1}) \\
			& \geq && \sum_{i\in [N]} \lambda \cdot \|(y_i-X_i \beta)\|_2^2 & (\text{Ineq.~\eqref{ineq:eigenvalue}}) \\
			& = && \lambda\cdot \psi^{(O)}(\beta), & (\text{Defns. of $\psi^{(O)}$})
		\end{split}
	\end{eqnarray*}
	which proves Inequality~\eqref{ineq:relation}.
	We also claim that for any $(i,t)\in [N]\times [T]$,
	\begin{align}
	\label{ineq:relation2}
	\psi^{(G,q)}_{it}(\zeta) \leq 2\cdot\left(\psi^{(O)}_{it} (\beta) +{\textstyle\sum}_{j=1}^{\min\left\{t-1,q\right\}} \psi^{(O)}_{i,t-j}(\beta)\right).
	\end{align}
	This trivially holds for $t=1$. 
	For $t\geq 2$, this is because
	\begin{eqnarray*}
		\begin{split}
			& && \psi^{(G,q)}_{it}(\zeta)&	\\
			&= && \left((y_{it}-x_{it}^\top \beta)-{\textstyle\sum}_{j=1}^{\min\left\{t-1,q\right\}} \rho_j\cdot (y_{i,t-j}-x_{i,t-j}^\top \beta)\right)^2 & (t\geq 2) \\
			& \leq && \left(1+{\textstyle\sum}_{j=1}^{\min\left\{t-1,q\right\}} \rho_j^2\right) \times \left((y_{it}-x_{it}^\top \beta)^2+{\textstyle\sum}_{j=1}^{\min\left\{t-1,q\right\}} (y_{i,t-j}-x_{i,t-j}^\top \beta)^2 \right) & (\text{Cauchy-Schwarz})  \\
			& =&& 2\cdot\left(\psi^{(O)}_{it} (\beta) + {\textstyle\sum}_{j=1}^{\min\left\{t-1,q\right\}} \psi^{(O)}_{i,t-j}(\beta)\right).  & (\|\rho\|_2^2\leq 1)
		\end{split}
	\end{eqnarray*}
	Now combining Inequalities~\eqref{ineq:relation} and~\eqref{ineq:relation2}, we have that for any $\zeta=(\beta,\rho)\in \calP_\lambda$,
	\begin{align*}
	\frac{\psi^{(G,q)}_{it}(\zeta)}{\psi^{(G,q)}(\zeta)} 
	&\leq &&\frac{2\cdot\left(\psi^{(O)}_{it} (\beta) + {\textstyle\sum}_{j=1}^{\min\left\{t-1,q\right\}} \psi^{(O)}_{i,t-j}(\beta)\right)}{\lambda \cdot \psi^{(O)}(\beta)} \\
	&\leq && 2\lambda^{-1}\cdot \left(s^{(O)}(i,t)+{\textstyle\sum}_{j=1}^{\min\left\{t-1,q\right\}} s^{(O)}(i,t-j) \right) \\
	&= && s(i,t).
	\end{align*}
	This completes the proof.
\end{proof}
	
	\section{Coresets for GLSE$_k$}
	\label{sec:glsek}
	
	Following from Section~\ref{sec:alg}, we assume that the parameter space is $\calP_\lambda^k = (\R^d \times B^q_{1-\lambda})^k$ for some given constant $\lambda\in (0,1)$.
	Given a panel dataset $X\in \R^{N\times T\times d}$ and $Y\in \R^{N\times T}$, let $Z^{(i)}\in \R^{T\times (d+1)}$ denote a matrix whose $t$-th row is $(x_{it},y_{it})\in \R^{d+1}$ for all $t\in [T]$ ($i\in [N]$).
	Assume there exists constant $M\geq 1$ such that the input dataset satisfies the following property.

	\begin{definition}[\bf{$M$-bounded dataset}]
		\label{def:bounded_dataset_main}
		Given $M\geq 1$, we say a panel dataset $X\in \R^{N\times T\times d}$ and $Y\in \R^{N\times T}$ is $M$-bounded if for any $i\in [N]$, the condition number of matrix $(Z^{(i)})^\top Z^{(i)}$ is at most $M$, i.e.,
		$
		\max_{\beta\in \R^d} \frac{\psi^{(O)}_i(\beta)}{\|\beta\|_2^2+1} \leq M\cdot \min_{\beta\in \R^d} \frac{\psi^{(O)}_i(\beta)}{\|\beta\|_2^2+1}.
		$
	\end{definition}

	\noindent
	If there exists $i\in [N]$ and $\beta\in \R^d$ such that $\psi^{(O)}_i(\beta)=0$, we let $M=\infty$.
	Specifically, if all $(Z^{(i)})^\top Z^{(i)}$ are identity matrix whose eigenvalues are all 1, i.e., for any $\beta$, $\psi^{(O)}_i(\beta) = \|\beta\|_2^2+1$, we can set $M=1$.
	Another example is that if $n\gg d$ and all elements of $Z^{(i)}$ are independently and identically distributed standard normal random variables, then the condition number of matrix $(Z^{(i)})^\top Z^{(i)}$ is upper bounded by some constant with high probability (and constant in expectation)~\cite{chen2005condition,shi2013the}, which may also imply $M = O(1)$.
	The main theorem is as follows.

	\begin{theorem}[\bf{Coresets for GLSE$_k$}]
		\label{thm:coreset_glsek}
		There exists a randomized algorithm that given an $M$-bounded ($M\geq 1$) panel dataset $X\in \R^{N\times T\times d}$ and $Y\in \R^{N\times T}$, constant $\eps,\lambda \in (0,1)$ and integers $q,k\geq 1$, with probability at least 0.9, constructs an $\eps$-coreset for GLSE$_k$ of size 
		\[
		O\left(\eps^{-4} \lambda^{-2} M k^2 \max\left\{q^7 d^4, q^5 d^6\right\} \cdot \log \frac{Mq}{ \lambda} \log \frac{Mkd}{\lambda} \right)
		\]
		and runs in time $O(NTq+NTd^2)$.
	\end{theorem}
	
	\noindent
	Similar to GLSE, this coreset for GLSE$_k$ ($k\geq 2$) contains at most
	\[
	(q+1)\cdot O\left(\eps^{-4} \lambda^{-2} M k^2 \max\left\{q^7 d^4, q^5 d^6\right\} \cdot \log \frac{Mq}{ \lambda} \log \frac{kd}{\lambda} \right)
	\]
	points $(x_{it},y_{it})$, which is independent of both $N$ and $T$ when $M$ is constant.
	Note that the size contains an addtional factor $M$ which can be unbounded.
	Our algorithm is summarized in Algorithm~\ref{alg:glsek} and we outline Algorithm~\ref{alg:glsek} and discuss the novelty in the following.

	\begin{algorithm}[htp!]
		\caption{$\CGk$: Coreset construction of GLSE$_k$}
		\label{alg:glsek}
		\begin{algorithmic}[1]
			\REQUIRE an $M$-bounded (constant $M\geq 1$) panel dataset $X\in \R^{N\times T\times d}$ and $Y\in \R^{N\times T}$, constant $\eps,\lambda \in (0,1)$, integers $k,q\geq 1$ and parameter space $\calP_\lambda^k$. \\
			\ENSURE a subset $S\subseteq [N]\times [T]$ together with a weight function $w:S\rightarrow \R_{\geq 0}$. \\
			\% {Constructing a subset of individuals}
			\STATE $ \Gamma \leftarrow O\left(\eps^{-2} \lambda^{-1} M k^2 \max\left\{q^4 d^2, q^3 d^3\right\}\cdot \log \frac{Mq}{\lambda}\right)$. 
			\STATE For each $i\in [N]$, let matrix $Z^{(i)}\in \R^{T\times (d+1)}$ be whose $t$-th row is $z^{(i)}_{t}=(x_{it},y_{it})\in \R^{d+1}$.
			\STATE For each $i\in [N]$, construct the SVD decomposition of $Z^{(i)}$ and compute
			\[
			u_i:=\lambda_{\max}((Z^{(i)})^\top Z^{(i)}) \text{ and }	\ell_i:=\lambda_{\min}((Z^{(i)})^\top Z^{(i)}).
			\]
			\STATE For each $i\in [N]$,
			$
			s^{(O)}(i)\leftarrow \frac{u_i}{u_i+\sum_{i'\neq i} \ell_{i'}}.
			$
			\STATE For each $i\in [N]$, $s(i)\leftarrow \min\left\{1, \frac{2(q+1)}{\lambda}\cdot s^{(O)}(i)\right\}$.
			\STATE Pick a random sample $I_S\subseteq [N]$ of size $M$, where each $i\in I_S$ is selected w.p. $\frac{s(i)}{\sum_{i'\in [N]}s(i')}$. 
			\STATE For each $i\in I_S$, $w'(i)\leftarrow \frac{\sum_{i'\in [N]}s(i')}{\Gamma \cdot s(i)}$. \\
			\% {Constructing a subset of time periods for each selected individual}
			\STATE For each $i\in I_S$, apply $\CG(X_i,y_i,\frac{\eps}{3},\frac{1}{20\Gamma},\lambda,q$) and construct $J_{S,i}\subseteq [T]$ together with a weight function $w^{(i)}: J_{S,i}\rightarrow \R_{\geq 0}$.
			\STATE Let $S\leftarrow \left\{(i,t)\in [N]\times [T]: i\in I_S, t\in J_{S,i}\right\}$. 
			\STATE For each $(i,t)\in S$, $w(i,t) \leftarrow w'(i)\cdot w^{(i)}(t)$.
			\STATE Output $(S,w)$.
		\end{algorithmic}
	\end{algorithm}

	\begin{remark}
		\label{remark:framework_glsek}
		Algorithm~\ref{alg:glsek} is a two-staged framework, which captures the $\min$ operations in GLSE$_k$. 

		\paragraph{First stage.} 
		We construct an $\frac{\eps}{3}$-coreset $I_S\subseteq [N]$ together with a weight function $w':I_S\rightarrow \R_{\geq 0}$ of the query space $(Z^{(G,q,k)},\calP^k,\psi^{(G,q,k)})$, i.e., for any $\zeta\in \calP^k$ 
		\[
		\sum_{i\in I_S} w'(i)\cdot \psi^{(G,q,k)}_{i}(\zeta) \in (1\pm \eps)\cdot \psi^{(G,q,k)}(\zeta).
		\]
		The idea is similar to Algorithm~\ref{alg:glse} except that we consider $N$ sub-functions $\psi^{(G,q,k)}_{i}$ instead of $NT$.
		In Lines 2-4 of Algorithm~\ref{alg:glsek}, we first construct a sensitivity function $s^{(O)}$ of OLSE$_k$.
		The definition of $s^{(O)}$ captures the impact of $\min$ operations in the objective function of OLSE$_k$ and the total sensitivity of $s^{(O)}$ is guaranteed to be upper bounded by Definition~\ref{def:bounded_dataset_main}.
		The key is showing that the maximum influence of individual $i$ is at most $\frac{u_i}{u_i+\sum_{j\neq i} \ell_j}$ (Lemma~\ref{lm:sen_olsek}), which implies that the total sensitivity of $s^{(O)}$ is at most $M$.
		Then in Line 5, we construct a sensitivity function $s$ of GLSE$_k$, based on a reduction from $s^{(O)}$ (Lemma~\ref{lm:sen_glsek}).

		\paragraph{Second stage.} In Line 8, for each $i\in I_S$, apply $\CG(X_i,y_i,\frac{\eps}{3},\frac{1}{20\cdot |I_S|},\lambda,q$) and construct a subset $J_{S,i}\subseteq [T]$ together with a weight function $w^{(i)}: J_{S,i}\rightarrow \R_{\geq 0}$.
		Output $S=\left\{(i,t)\in [N]\times [T]: i\in I_S, t\in J_{S,i}\right\}$ together with a weight function $w: S\rightarrow \R_{\geq 0}$ defined as follows: for any $(i,t)\in S$, $	w(i,t) := w'(i)\cdot w^{(i)}(t)$.
	\end{remark}
	
	\noindent
	We also provide a lower bound theorem which shows that the size of a coreset for GLSE$_k$ can be up to $\Omega(N)$.
	It indicates that the coreset size should contain additional factors than $k,q,d,1/\eps$, which reflects the reasonability of the $M$-bounded assumption.

	\begin{theorem}[\bf{Size lower bound of GLSE$_k$}]
		\label{thm:lower_main}
		Let $T=1$ and $d=k=2$ and $\lambda\in (0,1)$.
		There exists $X\in \R^{N\times T\times d}$ and $Y\in \R^{N\times T}$ such that any 0.5-coreset for GLSE$_k$ should have size $\Omega(N)$.
	\end{theorem}

\subsection{Proof overview}
\label{sec:proof_overview}

We first give a proof overview for summarization.

\paragraph{Proof overview of Theorem~\ref{thm:coreset_glsek}.}
For GLSE$_k$, we propose a two-staged framework (Algorithm~\ref{alg:glsek}): first sample a collection of individuals and then run $\CG$ on every selected individuals.
By Theorem~\ref{thm:coreset_glse}, each subset $J_{S,i}$ at the second stage is of size $\poly(q,d)$.
Hence, we only need to upper bound the size of $I_S$ at the first stage.
By a similar argument as that for GLSE, we can define the pseudo-dimension of GLSE$_k$ and upper bound it by $\poly(k,q,d)$, and hence, the main difficulty is to upper bound the total sensitivity of GLSE$_k$. 
We show that the gap between the individual regression objectives of GLSE$_k$ and OLSE$_k$ (GLSE$_k$ with $q=0$) with respect to the same $(\beta^{(1)},\ldots,\beta^{(k)})$ is at most $\frac{2(q+1)}{\lambda}$, which relies on $
\psi^{(G,q)}_i(\zeta) \geq \lambda \cdot \psi^{(O)}_i(\beta)$ and an observation that for any $\zeta=(\beta^{(1)},\ldots,\beta^{(k)},\rho^{(1)},\ldots,\rho^{(k)})\in \calP^k$,
$
\psi^{(G,q,k)}_i(\zeta) \leq 2(q+1)\cdot \min_{l\in [k]} \psi^{(O)}_{i}(\beta^{(l)}).
$
Thus, it suffices to provide an upper bound of the total sensitivity for OLSE$_k$. 
We claim that the maximum influence of individual $i$ is at most $\frac{u_i}{u_i+\sum_{j\neq i} \ell_j}$ where $u_i$ is the largest eigenvalue of $(Z^{(i)})^\top Z^{(i)}$ and $\ell_j$ is the smallest eigenvalue of $(Z^{(j)})^\top Z^{(j)}$.
This fact comes from the following observation:
$
\min_{l\in [k]}\|Z^{(i)} (\beta^{(l)},-1)\|_2^2 \leq \frac{u_i}{\ell_j}\cdot \min_{l\in [k]}\|Z^{(j)} (\beta^{(l)},-1)\|_2^2, 
$
and results in an upper bound $M$ of the total sensitivity for OLSE$_k$ since
$
\sum_{i\in [N]} \frac{u_i}{u_i+\sum_{j\neq i} \ell_j} \leq  \frac{\sum_{i\in [N]} u_i}{\sum_{j\in [N]} \ell_j} \leq M.
$

\paragraph{Proof overview of Theorem~\ref{thm:lower_main}.}
For GLSE$_k$, we provide a lower bound $\Omega(N)$ of the coreset size by constructing an instance in which any 0.5-coreset should contain observations from all individuals.
Note that we consider $T=1$ which reduces to an instance with cross-sectional data.
Our instance is to let $x_{i1}=(4^i,\frac{1}{4^i})$ and $y_{i1}=0$ for all $i\in [N]$.
Then letting $\zeta^{(i)}=(\beta^{(1)},\beta^{(2)},\rho^{(1)},\rho^{(2)})$ where $\beta^{(1)}=(\frac{1}{4^i},0)$, $\beta^{(2)}=(0,4^i)$ and $\rho^{(1)}=\rho^{(2)}=0$, we observe that $\psi^{(G,q,k)}(\zeta^{(i)})\approx \psi^{(G,q,k)}_i(\zeta^{(i)})$.
Hence, all individuals should be contained in the coreset such that regression objectives with respect to all $\zeta^{(i)}$ are approximately preserved.

\subsection{Proof of Theorem~\ref{thm:coreset_glsek}: Upper bound for GLSE$_k$}
\label{sec:proof_glsek_upper}

The proof of Theorem \ref{thm:coreset_glsek} relies on the following two theorems.
The first theorem shows that $I_S$ of Algorithm~\ref{alg:glsek} is an $\frac{\eps}{3}$-coreset  of $\left(Z^{G,q,k},\calP_\lambda^k,\psi^{(G,q,k)}\right)$.
The second one is a reduction theorem that for each individual in $I_S$ constructs an $\eps$-coreset $J_{S,i}$.

\begin{theorem}[\bf{Coresets of $\left(Z^{G,q,k},\calP_\lambda^k,\psi^{(G,q,k)}\right)$}]
	\label{thm:glsek_individual} 
	For any given $M$-bounded observation matrix $X\in \R^{N\times T\times d}$ and outcome matrix $Y\in \R^{N\times T}$, constant $\eps,\delta,\lambda \in (0,1)$ and integers $q,k\geq 1$, with probability at least 0.95, the weighted subset $I_S$ of Algorithm~\ref{alg:glsek} is an $\frac{\eps}{3}$-coreset of the query space $\left(Z^{G,q,k},\calP_\lambda^k,\psi^{(G,q,k)}\right)$, i.e., for any $\zeta=(\beta^{(1)},\ldots,\beta^{(k)},\rho^{(1)},\ldots,\rho^{(k)})\in \calP_\lambda^k$, 
	\begin{align}
	\label{ineq:I_S}
	\sum_{i\in I_S} w'(i)\cdot \psi^{(G,q,k)}_{i}(\zeta) \in (1\pm \frac{\eps}{3})\cdot \psi^{(G,q,k)}(\zeta).
	\end{align}
	Moreover, the construction time of $I_S$ is 
	\[
	N\cdot \SVD(T,d+1)+ O(N).
	\]
\end{theorem}

\noindent
We defer the proof of Theorem~\ref{thm:glsek_individual} later.

\begin{theorem}[\bf{Reduction from coresets of $\left(Z^{G,q,k},\calP_\lambda^k,\psi^{(G,q,k)}\right)$ to coresets for GLSE$_k$}]
	\label{thm:reduction}
	Suppose that the weighted subset $I_S$ of Algorithm~\ref{alg:glsek} is an $\frac{\eps}{3}$-coreset of the query space $\left(Z^{G,q,k},\calP_\lambda^k,\psi^{(G,q,k)}\right)$. 
	Then with probability at least 0.95, the output $(S,w)$ of Algorithm~\ref{alg:glsek} is an $\eps$-coreset for GLSE$_k$.
\end{theorem}

\begin{proof}[Proof of Theorem~\ref{thm:reduction}]
	Note that $S$ is an $\eps$-coreset for GLSE$_k$ if Inequality~\eqref{ineq:I_S} holds and for all $i\in [N]$, $J_{S,i}$ is an $\frac{\eps}{3}$-coreset of $\left((Z^{(i)})^{(G,q)},\calP_\lambda,\psi^{(G,q)}\right)$.
	By condition, we assume Inequality~\eqref{ineq:I_S} holds.
	By Line 6 of Algorithm~\ref{alg:glsek}, the probability that every $J_{S,i}$ is an $\frac{\eps}{3}$-coreset of $\left((Z^{(i)})^{(G,q)},\calP_\lambda,\psi^{(G,q)}\right)$ is at least
	\[
	1-\Gamma\cdot \frac{1}{20\Gamma} = 0.95,
	\]
	which completes the proof.
\end{proof}

\noindent
Observe that Theorem~\ref{thm:coreset_glsek} is a direct corollary of Theorems~\ref{thm:glsek_individual} and~\ref{thm:reduction}. 

\begin{proof}
	Combining Theorems~\ref{thm:glsek_individual} and~\ref{thm:reduction}, $S$ is an $\eps$-coreset of $\left(Z^{G,q,k},\calP_\lambda^k,\psi^{(G,q,k)}\right)$ with probability at least 0.9.
	By Theorem~\ref{thm:coreset_glse}, the size of $S$ is
	\[
	\Gamma \cdot O\left(\eps^{-2} \lambda^{-1} q d\left(\max\left\{q^2d, qd^2\right\}\cdot \log \frac{d}{\lambda}+\log \frac{\Gamma}{\delta}\right) \right),
	\] 
	which satisfies Theorem~\ref{thm:coreset_glsek} by pluging in the value of $\Gamma$.

	For the running time, it costs $N\cdot \SVD(T,d+1)$ to compute $I_S$ by Theorem~\ref{thm:glsek_individual}.
	Moreover, by Line 3 of Algorithm~\ref{alg:glsek}, we already have the SVD decomposition of $Z^{(i)}$ for all $i\in [N]$.
	Then it only costs $O\left(T(q+d)\right)$ to apply $\CG$ for each $i\in I_S$ in Line 8 of Algorithm~\ref{alg:glsek}.
	Then it costs $O\left(NT(q+d)\right)$ to construct $S$.
	This completes the proof of the running time.
\end{proof}

\paragraph{Proof of Theorem~\ref{thm:glsek_individual}: $I_S$ is a coreset of $\left(Z^{(G,q,k)},\calP_\lambda^k,\psi^{(G,q,k)}\right)$.}
It remains to prove Theorem~\ref{thm:glsek_individual}.
Note that the construction of $I_S$ applies the Feldman-Langberg framework.
The analysis is similar to Section~\ref{sec:alg} in which we provide upper bounds for both the total sensitivity and the pseudo-dimension.

We first discuss how to bound the total sensitivity of $(Z^{(G,q,k)},\calP^k,\psi^{(G,q,k)})$.
Similar to Section~\ref{sec:sen}, the idea is to first bound the total sensitivity of $(Z^{(G,0,k)},\calP^k,\psi^{(G,0,k)})$ -- we call it the query space of OLSE$_k$ whose covariance matrices of all individuals are identity matrices. 

\begin{lemma}[\bf{Total sensitivity of OLSE$_k$}]
	\label{lm:sen_olsek}
	Function $s^{(O)}:[N]\rightarrow \R_{\geq 0}$ of Algorithm~\ref{alg:glsek} satisfies that for any $i\in [N]$,
	\begin{align}
	\label{ineq:sen_olsek}
	s^{(O)}(i) \geq \sup_{\beta^{(1)},\ldots,\beta^{(k)}\in \R^d} \frac{\min_{l\in [k]}\psi^{(O)}_{i}(\beta^{(l)})}{\sum_{i'\in [N]}\min_{l\in [k]}\psi^{(O)}_{i'}(\beta^{(l)})},
	\end{align}
	and $\calG^{(O)} := \sum_{i\in [N]} s^{(O)}(i)$ satisfying that $\calG^{(O)} = O(M)$. 
	Moreover, the construction time of function $s^{(O)}$ is 
	\[
	N\cdot \SVD(T,d+1)+ O(N).
	\]
\end{lemma}

\begin{proof}
	For every $i\in [N]$, recall that $Z^{(i)}\in \R^{T\times (d+1)}$ is the matrix whose $t$-th row is $z^{(i)}_{t}=(x_{it},y_{it})\in \R^{d+1}$ for all $t\in [T]$.
	By definition, we have that for any $\beta\in \R^d$,
	\[
	\psi^{(O)}_{i}(\beta) = \|Z^{(i)} (\beta,-1)\|_2^2.
	\]
	Thus, by the same argument as in Lemma~\ref{lm:sen_olse}, it suffices to prove that for any matrix sequences $Z^{(1)},\ldots, Z^{(N)}\in \R^{T\times (d+1)}$,
	\begin{eqnarray}
	\label{ineq:sen_matrix}
	\begin{split}
	s^{(O)}(i) &\geq&& \sup_{\beta^{(1)},\ldots,\beta^{(k)}\in \R^{d}} \\
	& &&\frac{\min_{l\in [k]}\|Z^{(i)} (\beta^{(l)},-1)\|_2^2}{\sum_{i'\in [N]}\min_{l\in [k]}\|Z^{(i')} (\beta^{(l)},-1)\|_2^2}.
	\end{split}
	\end{eqnarray}
	For any $\beta^{(1)},\ldots,\beta^{(k)}\in \R^d$ and any $i\neq j\in [N]$, letting $l^\star = \arg\min_{l\in [k]} \|Z^{(j)} (\beta^{(l)},-1)\|_2^2$, we have
	\begin{align*}
	& && \min_{l\in [k]}\|Z^{(i)} (\beta^{(l)},-1)\|_2^2 &&\\
	& \leq && \|Z^{(i)} (\beta^{(l^\star)},-1)\|_2^2 && \\
	& \leq && u_i\cdot (\|\beta^{(l^\star)}\|_2^2+1) && (\text{Defn. of $u_i$}) \\
	& \leq && \frac{u_i}{\ell_j}\cdot \|Z^{(j)} (\beta^{(l^\star)},-1)\|_2^2 && (\text{Defn. of $\ell_i$}) \\
	& = && \frac{u_i}{\ell_j}\cdot \min_{l\in [k]}\|Z^{(j)} (\beta^{(l)},-1)\|_2^2. && (\text{Defn. of $l^\star$})
	\end{align*}
	Thus, we directly conclude that
	\begin{align*}
	& && \frac{\min_{l\in [k]}\|Z^{(i)} (\beta^{(l)},-1)\|_2^2}{\sum_{i'\in [N]}\min_{l\in [k]}\|Z^{(i')} (\beta^{(l)},-1)\|_2^2} &&\\
	& \leq && \frac{\min_{l\in [k]}\|Z^{(i)} (\beta^{(l)},-1)\|_2^2}{\left(1+\sum_{i'\neq i} \frac{\ell_{i'}}{u_i}\right)\cdot \min_{l\in [k]}\|Z^{(i)} (\beta^{(l)},-1)\|_2^2} && \\
	& = && \frac{u_i}{u_i+\sum_{i'\neq i} \ell_{i'}} && \\
	& = && s^{(O)}(i). &&
	\end{align*}
	Hence, Inequality~\eqref{ineq:sen_matrix} holds.
	Moreover, since the input dataset is $M$-bounded, we have
	\[
	\calG^{(O)} \leq \sum_{i\in [N]} \frac{u_i}{\sum_{i'\in [N]} \ell_{i'}} \leq M,
	\]
	which completes the proof of correctness.

	For the running time, it costs $N\cdot \SVD(T,d+1)$ to compute SVD decompositions for all $Z^{(i)}$.
	Then it costs $O(N)$ time to compute all $u_i$ and $\ell_i$, and hence costs $O(N)$ time to compute sensitivity functions $s^{(O)}$.
	Thus, we complete the proof.
\end{proof}

\noindent
Note that by the above argument, we can also assume
\[
\sum_{i\in [N]}\frac{u_i}{u_i+\sum_{i'\neq i} \ell_{i'}} \leq M,
\]
which leads to the same upper bound for the total sensitivity $\calG^{(O)}$.
Now we are ready to upper bound the total sensitivity of $(Z^{(G,q,k)},\calP^k,\psi^{(G,q,k)})$.

\begin{lemma}[\bf{Total sensitivity of GLSE$_k$}]
	\label{lm:sen_glsek}
	Function $s:[N]\rightarrow \R_{\geq 0}$ of Algorithm~\ref{alg:glsek} satisfies that for any $i\in [N]$,
	\begin{align}
	\label{ineq:sen_glsek}
	s(i) \geq \sup_{\zeta\in \calP_\lambda^k} \frac{\psi^{(G,q,k)}_{i}(\zeta)}{\psi^{(G,q,k)}(\zeta)},
	\end{align}
	and $\calG := \sum_{i\in [N]} s(i)$ satisfying that $\calG = O(\frac{qM}{\lambda})$.
	Moreover, the construction time of function $s$ is
	\[
	N\cdot \SVD(T,d+1)+ O(N).
	\]
\end{lemma}

\begin{proof}
	Since it only costs $O(N)$ time to construct function $s$ if we have $s^{(O)}$, we prove the construction time by Lemma~\ref{lm:sen_olsek}.

	Fix $i\in [N]$.
	If $s(i) = 1$ in Line 4 of Algorithm~\ref{alg:glsek}, then Inequality~\eqref{ineq:sen_glsek} trivally holds.
	Then we assume that $s(i) = \frac{2(q+1)}{\lambda}\cdot s^{(O)}(i)$. 
	We first have that for any $i\in [N]$ and any $\zeta\in \calP_\lambda^k$,
	\begin{align*}
	& && \psi^{(G,q,k)}_i(\zeta)&& \\
	& = && \min_{l\in [k]} {\textstyle\sum}_{t\in [T]} \psi^{(G,q)}_{it}(\beta^{(l)},\rho^{(l)}) && (\text{Defn.~\ref{def:glsek}}) \\
	& \geq && \min_{l\in [k]} {\textstyle\sum}_{t\in [T]} \lambda\cdot  \psi^{(O)}_{it}(\beta^{(l)}) && (\text{Ineq.~\eqref{ineq:relation}}) \\
	& = && \lambda\cdot \min_{l\in [k]} \psi^{(O)}_{i}(\beta^{(l)}). && (\text{Defn. of $\psi^{(O)}_i$})
	\end{align*}
	which directly implies that
	\begin{align}
	\label{ineq:sen1}
	\psi^{(G,q,k)}(\zeta) \geq \lambda \cdot \sum_{i'\in [N]}\min_{l\in [k]}\psi^{(O)}_{i'}(\beta^{(l)}).
	\end{align}
	We also note that for any $(i,t)\in [N]\times [T]$ and any $(\beta,\rho)\in \calP_\lambda$,
	\begin{eqnarray*}
		\begin{split}
			& && \psi^{(G,q)}_{it}(\beta,\rho) &\\
			& \leq  && \left((y_{it}-x_{it}^\top \beta)-{\textstyle\sum}_{j=1}^{\min\left\{t-1,q\right\}} \rho_j\cdot (y_{i,t-j}-x_{i,t-j}^\top \beta)\right)^2 & (\text{Defn. of $\psi^{(G,q)}_{it}$}) \\
			& \leq && (1+{\textstyle\sum}_{j=1}^{\min\left\{t-1,q\right\}} \rho_j^2) \times \left((y_{it}-x_{it}^\top \beta)^2+{\textstyle\sum}_{j=1}^{\min\left\{t-1,q\right\}} (y_{i,t-j}-x_{i,t-j}^\top \beta)^2\right) & (\text{Cauchy-Schwarz}) \\
			& \leq && 2\left((y_{it}-x_{it}^\top \beta)^2+{\textstyle\sum}_{j=1}^{\min\left\{t-1,q\right\}} (y_{i,t-j}-x_{i,t-j}^\top \beta)^2\right). & (\|\rho\|_2^2\leq 1)
		\end{split}
	\end{eqnarray*}
	Hence, we have that
	\begin{eqnarray}
	\label{ineq:sen2}
	\frac{1}{2}\cdot \psi^{(G,q)}_{it}(\beta,\rho) \leq (y_{it}-x_{it}^\top \beta)^2+\sum_{j=1}^{\min\left\{t-1,q\right\}} (y_{i,t-j}-x_{i,t-j}^\top \beta)^2.
	\end{eqnarray}
	This implies that
	\begin{eqnarray}
	\label{ineq:sen3}
	\begin{split}
	& && \psi^{(G,q,k)}_i(\zeta) & \\
	& = && \min_{l\in [k]} {\textstyle\sum}_{t\in [T]} \psi^{(G,q)}_{it}(\beta^{(l)},\rho^{(l)}) \quad &(\text{Defn.~\ref{def:glsek}}) \\
	& \leq && \min_{l\in [k]} {\textstyle\sum}_{t\in [T]} 2 \times \left((y_{it}-x_{it}^\top \beta)^2+{\textstyle\sum}_{j=1}^{\min\left\{t-1,q\right\}} (y_{i,t-j}-x_{i,t-j}^\top \beta)^2\right)
	& (\text{Ineq.~\eqref{ineq:sen2}})\\
	& \leq && 2(q+1)\cdot \min_{l\in [k]}{\textstyle\sum}_{t\in [T]} \psi^{(O)}_{it}(\beta^{(l)}) &\\
	& = && 2(q+1)\cdot \min_{l\in [k]} \psi^{(O)}_{i}(\beta^{(l)}). & (\text{Defn. of $\psi^{(O)}_i$})
	\end{split}
	\end{eqnarray}
	Thus, we have that for any $i\in [N]$ and $\zeta\in \calP_\lambda^k$,
	\begin{align*}
	\frac{\psi^{(G,q,k)}_{i}(\zeta)}{\psi^{(G,q,k)}(\zeta)} & \leq && \frac{2(q+1)\cdot \min_{l\in [k]} \psi^{(O)}_{i}(\beta^{(l)})}{\lambda\cdot \sum_{i\in [N]} \min_{l\in [k]} \psi^{(O)}_{i}(\beta^{(l)})} & (\text{Ineqs.~\eqref{ineq:sen1} and~\eqref{ineq:sen3}}) \\
	&\leq && \frac{2(q+1)}{\lambda} \cdot s^{(O)}(i) & (\text{Lemma~\ref{lm:sen_olsek}}) \\
	& = && s(i), & (\text{by assumption})
	\end{align*}
	which proves Inequality~\eqref{ineq:sen_glsek}.
	Moreover, we have that
	\[
	\calG = \sum_{i\in [N]} s(i) \leq \frac{2(q+1)}{\lambda} \cdot \calG^{(O)} = O(\frac{qM}{\lambda}),
	\]
	where the last inequality is from Lemma~\ref{lm:sen_olsek}.
	We complete the proof.
\end{proof}

\noindent
Next, we upper bound the pseudo-dimension of GLSE$_k$.
The proof is similar to that of GLSE by applying Lemmas~\ref{lm:dim_bound} and~\ref{lm:dim_range}.

\begin{lemma}[\bf{Pseudo-dimension of GLSE$_k$}]
	\label{lm:dim_glsek}
	\sloppy
	The pseudo-dimension of any query space $(Z^{(G,q,k)},u,\calP_\lambda^k,\psi^{(G,q,k)})$ over weight functions $u: [N]\rightarrow \R_{\geq 0}$ is at most 
	\[
	O\left(k^2q^2(q+d)d^2 \right).
	\]
\end{lemma}

\begin{proof}
	\sloppy
	The proof idea is similar to that of Lemma~\ref{lm:dim_glse}.
	Fix a weight function $u: [N]\rightarrow \R_{\geq 0}$.
	For every $i\in [N]$, let $g_{i}: \calP_\lambda^k\times \R_{\geq 0}\rightarrow \left\{0,1\right\}$ be the indicator function satisfying that for any $\zeta=(\beta^{(1)},\ldots,\beta^{(k)},\rho^{(1)},\ldots,\rho^{(k)})\in \calP_\lambda^k$ and $r\in \R_{\geq 0}$,
	\begin{align*}
	g_{i}(\zeta,r) &:= && I\left[u(i)\cdot\psi^{(G,q,k)}_{i}(\zeta)\geq r\right] \\
	&= && I\left[\forall l\in [k],~ u(i)\cdot \sum_{t\in [T]}\psi^{(G,q)}_{it}(\beta^{(l)},\rho^{(l)})\geq r\right].
	\end{align*}
	We consider the query space $(Z^{(G,q,k)},u,\calP_\lambda^k\times \R_{\geq 0},g)$.
	By the definition of $\calP_\lambda^k$, the dimension of $\calP_\lambda^k\times \R_{\geq 0}$ is $m=k(q+d)+1$.
	Also note that for any $(\beta,\rho)\in \calP_\lambda$, $\psi^{(G,q)}_{it}(\beta,\rho)$ can be represented as a multivariant polynomial that consists of $O(q^2 d^2)$ terms $\rho_{c_1}^{b_1} \rho_{c_2}^{b_2} \beta_{c_3}^{b_3} \beta_{c_4}^{b_4}$ where $c_1,c_2\in [q]$, $c_3,c_4\in [d]$ and $b_1,b_2,b_3,b_4\in \left\{0,1\right\}$.
	Thus, $g_{i}$ can be calculated using $l=O(kq^2d^2)$ operations, including $O(kq^2d^2)$ arithmetic operations and $k$ jumps.
	Pluging the values of $m$ and $l$ in Lemma~\ref{lm:dim_bound}, the pseudo-dimension of $(Z^{(G,q,k)},u,\calP_\lambda^k\times \R_{\geq 0},g)$ is $O\left(k^2q^2(q+d)d^2\right)$.
	Then by Lemma~\ref{lm:dim_range}, we complete the proof.
\end{proof}

\noindent
Combining with the above lemmas and Theorem~\ref{thm:fl11}, we are ready to prove Theorem~\ref{thm:glsek_individual}.

\begin{proof}[Proof of Theorem~\ref{thm:glsek_individual}]
	By Lemma~\ref{lm:sen_glsek}, the total sensitivity $\calG$ of $(Z^{(G,q,k)},\calP_\lambda^k,\psi^{(G,q,k)})$ is $O(\frac{qM}{\lambda})$.
	By Lemma~\ref{lm:dim_glsek}, we can let $\dim = O\left(k^2(q+d)q^2d^2 \right)$ which is an upper bound of the pseudo-dimension of every query space $(Z^{(G,q,k)},u,\calP_\lambda^k,\psi^{(G,q,k)})$ over weight functions $u: [N]\rightarrow \R_{\geq 0}$.
	Pluging the values of $\calG$ and $\dim$ in Theorem~\ref{thm:fl11}, we prove for the coreset size.

	For the running time, it costs $N\cdot \SVD(T,d+1)+ O(N)$ time to compute the sensitivity function $s$ by Lemma~\ref{lm:sen_glsek}, and $O(N)$ time to construct $I_S$.
	This completes the proof.
\end{proof}

\subsection{Proof of Theorem~\ref{thm:lower_main}: Lower bound for GLSE$_k$}
\label{sec:lower}

Actually, we prove a stronger version of Theorem~\ref{thm:lower_main} in the following.
We show that both the coreset size and the total sensitivity of the query space $(Z^{(G,q,k)},u,\calP_\lambda^k,\psi^{(G,q,k)})$ may be $\Omega(N)$, even for the simple case that $T=1$ and $d=k=2$.

\begin{theorem}[\bf{Size and sensitivity lower bound of GLSE$_k$}]
	\label{thm:lower}
	Let $T=1$ and $d=k=2$ and $\lambda\in (0,1)$.
	There exists an instance $X\in \R^{N\times T\times d}$ and $Y\in \R^{N\times T}$ such that the total sensitivity
	\[
	\sum_{i\in [N]} \sup_{\zeta\in \calP_\lambda^k} \frac{\psi^{(G,q,k)}_i(\zeta)}{\psi^{(G,q,k)}(\zeta)} = \Omega(N).
	\]
	and any 0.5-coreset of the query space $(Z^{(G,q,k)},u,\calP_\lambda^k,\psi^{(G,q,k)})$ should have size $\Omega(N)$.
\end{theorem}

\begin{proof}
	We construct the same instance as in~\cite{Tolochinsky2018GenericCF}. 
	Concretely, for $i\in [N]$, let $x_{i1} = (4^i,\frac{1}{4^i})$ and $y_{i1}=0$.
	We claim that for any $i\in [N]$,
	\begin{align}
	\label{ineq:sen_fraction}
	\sup_{\zeta\in \calP_\lambda^k} \frac{\psi^{(G,q,k)}_i(\zeta)}{\psi^{(G,q,k)}(\zeta)} \geq \frac{1}{2}.
	\end{align}
	If the claim is true, then we complete the proof of the total sensitivity by summing up the above inequality over all $i\in [N]$.
	Fix $i\in [N]$ and consider the following $\zeta=(\beta^{(1)},\beta^{(2)},\rho^{(1)},\rho^{(2)})\in \calP_\lambda^k$ where $\beta^{(1)}=(\frac{1}{4^i},0)$, $\beta^{(2)}=(0,4^i)$ and $\rho^{(1)}=\rho^{(2)}=0$.
	If $j\leq i$, we have
	\begin{align*}
	\psi^{(G,q,k)}_j(\zeta) &=&& \min_{l\in [2]} (y_{i1}-x_{i1}^\top \beta^{(l)})^2 \\
	&= && \min\left\{\frac{1}{16^{j-i}}, \frac{1}{16^{i-j}}\right\} \\
	& =&& \frac{1}{16^{i-j}}.
	\end{align*}
	Similarly, if $j>i$, we have
	\[
	\psi^{(G,q,k)}_j(\zeta) = \min\left\{\frac{1}{16^{j-i}}, \frac{1}{16^{i-j}}\right\} = \frac{1}{16^{j-i}}.
	\]
	By the above equations, we have
	\begin{align}
	\label{eq:lower1}
	\psi^{(G,q,k)}(\zeta) = \sum_{j=1}^{i} \frac{1}{16^{i-j}} + \sum_{j=i+1}^{N} \frac{1}{16^{j-i}} < \frac{5}{4}.
	\end{align}
	Combining with the fact that $\psi^{(G,q,k)}_i(\zeta)=1$, we prove Inequality~\eqref{ineq:sen_fraction}.

	For the coreset size, suppose $S\subseteq [N]$ together with a weight function $w:S\rightarrow \R_{\geq 0}$ is a 0.5-coreset of the query space $(Z^{(G,q,k)},u,\calP_\lambda^k,\psi^{(G,q,k)})$.
	We only need to prove that $S=[N]$.
	Suppose there exists some $i^\star\in S$ with $w(i^\star) > 2$.
	Letting $\zeta=(\beta^{(1)},\beta^{(2)},\rho^{(1)},\rho^{(2)})$ where $\beta^{(1)}=(\frac{1}{4^{i^\star}},0)$, $\beta^{(2)}=(0,4^{i^\star})$ and $\rho^{(1)}=\rho^{(2)}=0$, we have that
	\begin{align*}
	\sum_{i\in S} w(i)\cdot \psi^{(G,q,k)}_i(\zeta) & > && w(i^\star)\cdot \psi^{(G,q,k)}_{i^\star}(\zeta) &\\
	& > && 2 & (w(i^\star)>2 \text{ and Defns. of $\zeta$}) &\\
	& > && (1+\frac{1}{2})\cdot \frac{5}{4} &\\
	& > && (1+\frac{1}{2})\cdot \psi^{(G,q,k)}(\zeta), & (\text{Ineq.~\eqref{eq:lower1}})
	\end{align*}
	which contradicts with the assumption of $S$.
	Thus, we have that for any $i\in S$, $w(i)\leq 2$.
	Next, by contradiction assume that $i^\star\notin S$.
	Again, letting $\zeta=(\beta^{(1)},\beta^{(2)},\rho^{(1)},\rho^{(2)})$ where $\beta^{(1)}=(\frac{1}{4^{i^\star}},0)$, $\beta^{(2)}=(0,4^{i^\star})$ and $\rho^{(1)}=\rho^{(2)}=0$, we have that
	\begin{align*}
	\sum_{i\in S} w(i)\cdot \psi^{(G,q,k)}_i(\zeta) & \leq && 2\left(\psi^{(G,q,k)}(\zeta)- \psi^{(G,q,k)}_{i^\star}(\zeta)\right) &\\
	& && (w(i)\leq 2) &\\
	& \leq && 2 (\frac{5}{4}-1) & (\text{Ineq.~\eqref{eq:lower1}}) \\
	& \leq && (1-\frac{1}{2})\cdot 1 &\\
	& \leq && (1-\frac{1}{2})\cdot \psi^{(G,q,k)}(\zeta), &
	\end{align*}	
	which contradicts with the assumption of $S$.
	This completes the proof.

\end{proof}

\section{Empirical results}
\label{sec:empirical}

We implement our coreset algorithms for GLSE, and compare the performance with uniform sampling on synthetic datasets and a real-world dataset.
The experiments are conducted by PyCharm on a 4-Core desktop CPU with 8GB RAM.\footnote{Codes are in \url{https://github.com/huanglx12/Coresets-for-regressions-with-panel-data}.}

\noindent\textbf{Datasets.}
We experiment using \textbf{synthetic} datasets with $N=T=500$ ($250k$ observations), $d=10$, $q=1$ and $\lambda = 0.2$.
For each individual $i\in [N]$, we first generate a mean vector $\overline{x}_i\in \R^{d}$ by first uniformly sampling a unit vector $x'_i\in \R^d$, and a length $\tau\in [0,5]$, and then letting $\overline{x}_i = \tau x'_i$.
Then for each time period $t\in [T]$, we generate observation $x_{it}$ from a multivariate normal distribution $N(\overline{x}_i, \|\overline{x}_i\|_2^2\cdot I)$~\cite{tong2012multivariate}.\footnote{The assumption that the covariance of each     individual is proportional to $\|\overline{x}_i\|_2^2$ is common in econometrics.
	We also fix the last coordinate of $x_{it}$ to be 1 to capture individual specific fixed effects.}
Next, we generate outcomes $Y$.
First, we generate a regression vector $\beta\in \R^d$ from distribution $N(0,I)$. 
Then we generate an autoregression vector $\rho\in \R^q$ by first uniformly sampling a unit vector $\rho'\in \R^q$ and a length $\tau\in [0,1-\lambda]$, and then letting $\rho = \tau \rho'$.
Based on $\rho$, we generate error terms $e_{it}$ as in Equation~\eqref{eq:error}.  To assess performance robustness in the presence of outliers,  we  simulate another dataset replacing  $N(0,I)$  in Equation~\eqref{eq:error} with the heavy tailed \textbf{Cauchy}(0,2) distribution~\cite{ma2014statistical}.
Finally, the outcome $y_{it} = x_{it}^\top \beta+e_{it}$  is the same as Equation~\eqref{eq:linear}.

We also experiment on a \textbf{real-world} dataset involving the prediction of monthly profits from customers for a credit card issuer as a function of demographics, past behaviors, and current balances and fees. 
The panel dataset consisted of 250k observations: 50 months of data ($T=50$) from 5000 customers ($N=5000$) with 11 features ($d=11$). 
We set $q=1$ and $\lambda = 0.2$.

\noindent\textbf{Baseline and metrics.}
As a baseline coreset, we use uniform sampling (\textbf{Uni}), perhaps the simplest approach to construct coresets: Given an integer $\Gamma$, uniformly sample $\Gamma$ individual-time pairs $(i,t)\in [N]\times [T]$ with weight $\frac{NT}{\Gamma}$ for each.

Given regression parameters $\zeta$ and a subset $S\subseteq [N]\times [T]$, we define the \emph{empirical error} as $\left| \frac{\psi^{(G,q)}_S(\zeta)}{\psi^{(G,q)}(\zeta)}-1 \right|$.
We summarize the empirical errors $e_1,\ldots, e_n$ by maximum, average, standard deviation (std) and root mean square error (RMSE), where RMSE$= \sqrt{\frac{1}{n}\sum_{i\in [n]}e_i^2}$.  By penalizing larger errors,  RMSE combines information in both average and standard deviation as a performance metric,.
The running time for solving GLSE on dataset $X$ and our coreset $S$ are $T_X$ and $T_S$ respectively. $T_C$ is the running time for coreset $S$ construction .

\noindent\textbf{Simulation setup.}
We vary $\eps = 0.1, 0.2, 0.3, 0.4, 0.5$ and generate 100 independent random tuples $\zeta=(\beta,\rho)\in \R^{d+q}$ (the same as described in the generation of the synthetic dataset).
For each $\eps$, we run our algorithm $\CG$ and \textbf{Uni} to generate coresets.
We guarantee that the total number of sampled individual-time pairs of $\CG$ and \textbf{Uni} are the same.
We also implement IRLS~\cite{jorgensen2006iteratively} for solving GLSE. 
We run IRLS on both the full dataset and coresets and record the runtime. 

\noindent\textbf{Results.}
Table~\ref{tab:glse} summarizes the accuracy-size trade-off of our coresets for GLSE  for different error guarantees $\eps$.
The maximum empirical error of \textbf{Uni} is always larger than that of our coresets (1.16-793x). Further, there is no error guarantee with \textbf{Uni}, but errors are always below the error guarantee with our coresets. The speed-up with our coresets relative to full data ($\frac{T_X}{T_C+T_S}$) in solving GLSE is 1.2x-108x.
To achieve the maximum empirical error of .294 for GLSE in the real-world data, only 1534 individual-time pairs (0.6\%) are necessary for $\CG$.  With \textbf{Uni}, to get the closest maximum empirical error of 0.438, at least 2734 individual-time pairs) (1.1\%) is needed; i.e..,  $\CG$ achieves a smaller empirical error with a smaller sized coreset.
Though \textbf{Uni} may sometimes provide  lower average error than $\CG$, it \textit{always} has higher RMSE, say 1.2-745x of $\CG$.
When there are outliers as with Cauchy, our coresets perform even better on all metrics relative to  \textbf{Uni}.
This is because $\CG$ captures tails/outliers in the coreset, while \textbf{Uni} does not. 
Figure~\ref{fig:boxplot} presents the boxplots of the empirical errors.

\begin{table}[t]
	\centering
	\caption{performance of $\eps$-coresets for GLSE w.r.t. varying $\eps$. We report the maximum/average/standard deviation/RMSE of the empirical error w.r.t. the 100 tuples of generated regression parameters for our algorithm $\CG$ and \textbf{Uni}.
		Size is the  \# of sampled individual-time pairs, for both $\CG$ and \textbf{Uni}. $T_C$ is construction time (seconds) of our coresets. $T_S$ and $T_X$ are the computation time (seconds) for GLSE over coresets and the full dataset respectively.
		``Synthetic (G)'' and ``Synthetic (C)'' represent synthetic datasets with Gaussian errors and Cauchy errors respectively.
	}
	
	\label{tab:glse}
	\begin{tabular}{ccccccrccc}
		\toprule
		& \multirow{2}{*}{$\eps$} & \multicolumn{2}{c}{max. emp. err.} & \multicolumn{2}{c}{avg./std./RMSE of emp. err.} & \multirow{2}{*}{size} &  \multirow{2}{*}{$T_C$} & \multirow{2}{*}{$T_C + T_S$} & \multirow{2}{*}{$T_X$ (s)} \\ 
		& & $\CG$ & \textbf{Uni} & $\CG$ & \textbf{Uni} & & & & \\
		\midrule
		\multirow{5}{*}{\rotatebox[origin=c]{90}{synthetic (G)}} 
		& 0.1 & \textbf{.005} & .015 & .001/.001/.002 & .007/.004/.008 & 116481 & 2 & 372 & 458 \\
		& 0.2 & \textbf{.018} & .029 & .006/.004/.008 & .010/.007/.013 & 23043 & 2  & 80  & 458\\
		& 0.3 & \textbf{.036} & .041 & .011/.008/.014 & .014/.010/.017 & 7217 & 2  &  29 & 458\\
		& 0.4 & \textbf{.055} & .086 & .016/.012/.021 & .026/.020/.032 & 3095 &  2 &  18 & 458\\
		& 0.5 & \textbf{.064} & .130 & .019/.015/.024 & .068/.032/.075 & 1590 &  2 &  9 & 458\\
		\midrule
		\multirow{5}{*}{\rotatebox[origin=c]{90}{synthetic (C)}} 
		& 0.1 & \textbf{.001} & .793 & .000/.000/.001 & .744/.029/.745 & 106385 & 2 & 1716 & 4430\\
		& 0.2 & \textbf{.018} & .939 & .013/.003/.014 & .927/.007/.927 & 21047 & 2  & 346  & 4430\\
		& 0.3 & \textbf{.102} & .937 & .072/.021/.075 & .860/.055/.862 & 6597 & 2  &  169 & 4430\\
		& 0.4 & \textbf{.070} & .962 & .051/.011/.053 & .961/.001/.961 & 2851 &  2 &  54 & 4430\\
		& 0.5 & \textbf{.096} & .998 & .060/.026/.065 & .992/.004/.992 & 472 &  2 &  41 & 4430\\
		\midrule
		\multirow{5}{*}{\rotatebox[origin=c]{90}{real-world}} 
		& 0.1 & \textbf{.029} & .162 & .005/.008/.009 & .016/.026/.031 & 50777 & 3 & 383  & 2488 \\
		& 0.2 & \textbf{.054} & .154 & .017/.004/.017 & .012/.024/.026 & 13062 & 3  & 85  & 2488 \\
		& 0.3 & \textbf{.187} & .698 & .039/.038/.054 & .052/.106/.118 & 5393 & 3  & 24  & 2488\\
		& 0.4 & \textbf{.220} & .438 & .019/.033/.038 & .050/.081/.095 & 2734 &  3 & 20  & 2488\\
		& 0.5 & \textbf{.294} & 1.107 & .075/.038/.084 & .074/.017/.183 & 1534 &  3 & 16 & 2488\\
		\bottomrule
	\end{tabular}
	
\end{table}

\begin{figure}
	\includegraphics[width = 0.48\textwidth]{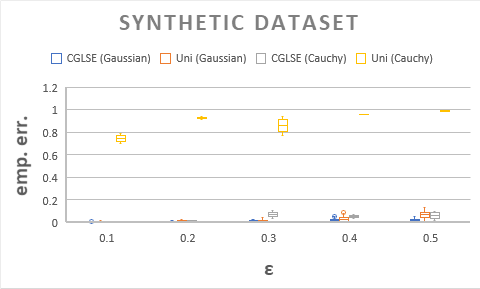}
	\quad
	\includegraphics[width = 0.48\textwidth]{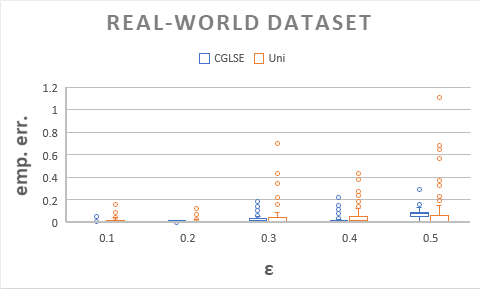}
	\caption{Boxplots of empirical errors for GLSE w.r.t. varying $\eps$. \textbf{Uni} has higher average and maximum empirical errors than $\CG$.}
	\label{fig:boxplot}
\end{figure}

\section{Conclusion, limitations, and future work}
\label{sec:conclusion}

This paper initiates a theoretical study of coreset construction for regression problems with panel data.
We formulate the definitions of coresets for several variants of $\ell_2$-regression, including OLSE, GLSE, and GLSE$_k$.
For each variant, we propose efficient algorithms that construct a coreset of size independent of both $N$ and $T$, based on the FL framework.
Our empirical results indicate that our algorithms can accelerate the evaluation time and perform significantly better than uniform sampling.

For GLSE$_k$, our coreset size contains a factor $M$, which may be unbounded and result in a coreset of size $\Omega(N)$ in the worst case.
In practice, if $M$ is large, each sensitivity $s(i)$ in Line 5 of Algorithm~\ref{alg:glsek} will be close or even equal to 1.
In this case, $I_S$ is drawn from all individuals via uniform sampling which weakens the performance of Algorithm~\ref{alg:glsek} relative to \textbf{Uni}.
Future research should investigate whether a different assumption than the $M$-bound can generate a coreset of a smaller size.

There are several directions for future work.
Currenly,  $q$ and $d$ have a relatively large impact on coreset size; future work needs to reduce this effect. This  will advance the use of coresets for machine learning, where  $d$ is typically large, and $q$ is large in high frequency data. 
This paper focused on coreset construction for panel data with $\ell_2$-regression. The natural next steps would be to construct coresets with panel data for other regression problems, e.g., $\ell_1$-regression, generalized linear models and logistic regression, and beyond regression to other supervised machine learning algorithms.

\vspace{-5mm}
\paragraph{Broader impact.} In terms of  broader impact on practice, many organizations have to routinely outsource data processing to external consultants and statisticians. 
But a major practical challenge for organizations in doing this is to minimize issues of data security in terms of exposure of their data for potential abuse. 
Further, minimization of such exposure is considered as necessary due diligence by laws such as GDPR and CCPA which mandates firms to minimize security breaches that violate the privacy rights of the data owner \cite{shastri2019seven, ke2020privacy}.
Coreset based approaches to sharing data for processing can be very valuable for firms in addressing data security and to be in compliance with privacy regulations like GDPR and CCPA.

Further, for policy and managerial decision making in economics, social sciences and management, obtaining unbiased estimates of the regression relationships from observational data is critical. 
Panel data is a critical ingredient for obtaining such unbiased estimates. 
As ML methods are being adopted by many social scientists \cite{athey2015machine}, ML scholars are becoming sensitive to these issues and our work in using coreset methods for panel data can have significant impact for these scholars. 
A practical concern is that coresets constructed and shared for one purpose or model may be used by the data processor for other kinds of models, which may lead to erroneous conclusions. Further, there is also the potential for issues of fairness to arise as different groups may not be adequately represented in the coreset without incorporating fairness constraints \cite{huang2019coresets}.
These issues need to be explored in future research.

\section*{Acknowledgements}

This research was conducted when LH was at Yale and  was supported in part by an NSF CCF-1908347 grant.

\bibliography{references}
\bibliographystyle{plain}

\newpage
\appendix

\section{Discussion of the generative model~\eqref{eq:linear}}
\label{sec:discussion}

In this section, we discuss the equivalence between the generative model~\eqref{eq:linear} and the random effects estimator.
In random effects estimators, there exist additional individual specific effects $\alpha_i\in \R$, i.e.,
\begin{align}
\label{eq:constant}
y_{it} = x_{it}^\top \beta_i + \alpha_i + e_{it},
\end{align}
and we assume that all individual effects are drawn from a normal distribution, i.e.,
\[
\alpha_i \sim N(\mu,\sigma_0^2), \qquad \forall i\in [N].
\]
where $\mu\in \R$ is the mean and $\sigma_0^2\in \R_{\geq 0}$ is the covariance of an unknown normal distribution.
By Equation~\eqref{eq:constant}, for any $i\in [N]$, we let $\alpha_i = \mu + \eps_i$ where $\eps_i\sim N(0,\sigma_0^2)$.
Then Equation~\eqref{eq:constant} can be rewritten as
\[
y_{it} = x_{it}^\top \beta_i + \mu +(\eps_i + e_{it}).
\]
Let $\Omega\in \R^{T\times T}$ denote the covariance matrix among error terms $e_{it}$.
Next, we simplify $\eps_i+e_{it}$ by $e'_{it}$.
Consequently, error terms $e'_{it}$ satisfy that
\begin{eqnarray*}
	\label{eq:cov_combine}
	\begin{split}
		&\Exp[e'_{it}]=0, && \forall (i,t)\in [N]\times [T]; && \\
		& \Cov(e'_{it},e'_{i' t'}) = 0 && \forall i\neq i' && \\
		&\Cov(e'_{it},e'_{i t'}) = \Omega_{t t'} + \sigma_0^2 = \Omega'_{t t'} && \forall i\in [N], t,t'\in [T]. &&
	\end{split}
\end{eqnarray*}
By this assumption, a random effects estimator can be defined by the following:
\begin{eqnarray*}
	\label{prog:regression}
	\min_{\beta, \Omega} \sum_{i\in [N]}  (y_i-X_i \beta_i-\mu\cdot \mathbf{1})^\top (\Omega')^{-1} (y_i-X_i \beta_i-\mu\cdot \mathbf{1}). 
\end{eqnarray*}
Thus, we verify that the random effects estimator is equivalent to the generative model~\eqref{eq:linear}.

\section{Existing results and approaches for OLSE}
\label{sec:existing}

We note that finding an $\eps$-coreset of $X$ for OLSE can be reduced to finding an $\eps$-coreset for least-squares regression with cross-sectional data.
For completeness, we summarize the following theorems for OLSE whose proofs mainly follow from the literature.

\begin{theorem}[\bf{$\eps$-Coresets for OLSE~\cite{boutsidis2013near}}]
	\label{thm:olse}
	There exists a deterministic algorithm that for any given observation matrix $X\in \R^{N\times T\times d}$, outcome matrix $Y\in \R^{N\times T}$, a collection $\calB\subseteq \R^d$ and constant $\eps\in (0,1)$, constructs an $\eps$-coreset of size $O(d/\eps^2)$ of OLSE, with running time $T_{SVD} + O(NTd^3/\eps^2)$ where $T_{SVD}$ is the time needed to compute the left singular vectors of a matrix in $\R^{NT\times (d+1)}$.
\end{theorem}

\begin{theorem}[\bf{Accurate coresets for OLSE~\cite{jubran2019fast}}]
	\label{thm:OLS_acc}
	There exists a deterministic algorithm that for any given observation matrix $X\in \R^{N\times T\times d}$, outcome matrix $Y\in \R^{N\times T}$, a collection $\calB\subseteq \R^d$, constructs an accurate coreset of size $O(d^2)$ of OLSE, with running time $O(NTd^2+d^8 \log (NT/d))$.
\end{theorem}

\subsection{Proof of Theorem~\ref{thm:olse}}
\label{sec:proof1}

We first prove Theorem~\ref{thm:olse} and propose the corresponding algorithm that constructs an $\eps$-coreset.
Recall that $\calB\subseteq \R^d$ denotes the domain of possible vectors $\beta$. 

\begin{proof}[Proof of Theorem~\ref{thm:olse}]
	Construct a matrix $A\in \R^{NT\times d}$ by letting the $(iT-T+t)$-th row of $A$ be $x_{it}$ for $(i,t)\in [N]\times [T]$.
	Similarly, construct a vector $\mathbf{b}\in \R^{NT}$ by letting $\mathbf{b}_{iT-T+t} = y_{it}$.
	Then for any $\beta\in \calB$, we have
	\[
	\psi^{(O)}(\beta) = \|A\beta - \mathbf{b}\|_2^2.
	\]
	Thus, finding an $\eps$-coreset of $X$ of OLSE is equivalent to finding a row-sampling matrix $S\in \R^{m\times NT}$ whose rows are basis vectors $e_{i_1}^\top,\ldots, e_{i_m}^\top$ and a rescaling matrix $W\in \R_{\geq 0}^{m\times m}$ that is a diagonal matrix such that for any $\beta\in \calB$,
	\[
	\|WS\left(A\beta-\mathbf{b} \right) \|_2^2\in (1\pm \eps)\cdot \|A\beta - \mathbf{b}\|_2^2.
	\]
	By Theorem 1 of~\cite{boutsidis2013near}, we only need $m=O(d/\eps^2)$ which completes the proof of correctness.
	Note that Theorem 1 of~\cite{boutsidis2013near} only provides a theoretical guarantee of a weak-coreset which only approximately preserves the optimal least-squares value.
	However, by the proof of Theorem 1 of~\cite{boutsidis2013near}, their coreset indeed holds for any $\beta\in \R^d$.

	The running time also follows from Theorem 1 of~\cite{boutsidis2013near}, which can be directly obtained by the algorithm stated below. 
\end{proof}

\paragraph{Algorithm in~\cite{boutsidis2013near}.}
We then introduce the approach of~\cite{boutsidis2013near} as follows.
Suppose we have inputs $A\in \R^{n\times d}$ and $\mathbf{b}\in \R^n$.
\begin{enumerate}
	\item Compute the SVD of $Y=[A,b]\in \R^{n\times (d+1)}$.
	Let $Y= U\Sigma V^\top$ where $U\in \R^{n\times (d+1)}, \Sigma\in \R^{(d+1)\times (d+1)}$ and $V\in \R^{(d+1)\times (d+1)}$.
	\item By Lemma 2 of~\cite{boutsidis2013near} which is based on Theorem 3.1 of~\cite{batson2012twice}, we deterministically construct sampling and rescaling matrices $S\in \R^{m\times n}$ and $W\in \R^{m\times m}$ ($m=O(d/\eps^2)$) such that for any $y\in \R^{d+1}$,
	\[
	\|WSUy\|_2^2\in (1\pm \eps)\cdot \|Uy\|_2^2.
	\]
	The construction time is $O(nd^3/\eps^2)$.
	\item Output $S$ and $W$.
\end{enumerate}

\subsection{Proof of Theorem~\ref{thm:OLS_acc}}
\label{sec:proof2}

Next, we prove Theorem~\ref{thm:OLS_acc} and propose the corresponding algorithm that constructs an accurate coreset.

\begin{proof}[Proof of Theorem~\ref{thm:OLS_acc}]
	The proof idea is similar to that of Theorem~\ref{thm:olse}.
	Again, we construct a matrix $A\in \R^{NT\times d}$ by letting the $(iT-T+t)$-th row of $A$ be $x_{it}$ for $(i,t)\in [N]\times [T]$.
	Similarly, construct a vector $\mathbf{b}\in \R^{NT}$ by letting $\mathbf{b}_{iT-T+t} = y_{it}$.
	Then for any $\beta\in \calB$, we have
	\[
	\psi^{(O)}(\beta) = \|A\beta - \mathbf{b}\|_2^2.
	\]
	Thus, finding an $\eps$-coreset of $X$ of OLSE is equivalent to finding a row-sampling matrix $S\in \R^{m\times NT}$ whose rows are basis vectors $e_{i_1}^\top,\ldots, e_{i_m}^\top$ and a rescaling matrix $W\in \R_{\geq 0}^{m\times m}$ that is a diagonal matrix such that for any $\beta\in \calB$,
	\[
	\|WS\left(A\beta-\mathbf{b} \right) \|_2^2 = \|A\beta - \mathbf{b}\|_2^2.
	\]
	By Theorem 3.2 of~\cite{jubran2019fast}, we only need $m=(d+1)^2+1 = O(d^2)$. 
	Moreover, we can construct matrices $W$ and $S$ in $O(NTd^2+d^8 \log (NT/d))$ time by applying $n= NT$, and $k=2(d+1)$ in Theorem 3.2 of~\cite{jubran2019fast}.
\end{proof}

\paragraph{Main approach in~\cite{jubran2019fast}.} 
Suppose we have inputs $A\in \R^{n\times d}$ and $\mathbf{b}\in \R^n$.
Let $A'=[A, \mathbf{b}]\in \R^{n\times (d+1)}$
For any $\beta\in \R^d$, we let $\beta' = (\beta, -1)\in \R^{d+1}$ and have that
\[
\|A\beta-\mathbf{b}\|_2^2 = \|A' \beta'\|_2^2 = (\beta')^\top (A')^\top A' \beta'.
\]
The main idea of~\cite{jubran2019fast} is to construct a sub-matrix $C\in \R^{\left((d+1)^2+1\right)\times (d+1)}$ of $A'$ whose rows are of the form $w_i\cdot (a_i, \mathbf{b}_i)^\top$ for some $i\in [n]$ and $w_i\geq 0$, such that $C^\top C = (A')^\top A'$.
Then we have for any $\beta\in \R^d$,
\[
\|C \beta'\|_2^2 = (\beta')^\top C^\top C \beta' = (\beta')^\top (A')^\top A' \beta' = \|A\beta-\mathbf{b}\|_2^2.
\]
By the definition of $C$, there exists a row-sampling matrix $S$ and a rescaling matrix $W$ such that $C=WSA'$.

We then discuss why such a sub-matrix $C$ exists.
The main observation is that $(A')^\top A'\in \R^{(d+1)\times (d+1)}$ and
\[
(A')^\top A' = \sum_{i\in [n]} (a_i, \mathbf{b}_i)\cdot (a_i, \mathbf{b}_i)^\top.
\]
Thus, $\frac{1}{n}\cdot (A')^\top A'$ is inside the convex hull of $n$ matrices $(a_i, \mathbf{b}_i)\cdot (a_i, \mathbf{b}_i)^\top\in \R^{(d+1)\times (d+1)}$.
By the Caratheodory's Theorem, there must exist at most $(d+1)^2+1$ matrices $(a_i, \mathbf{b}_i)\cdot (a_i, \mathbf{b}_i)^\top$ whose convex hull also contains $\frac{1}{n}\cdot (A')^\top A'$.
Then $\frac{1}{n}\cdot (A')^\top A'$ can be represented as a linear combination of these matrices, and hence, the sub-matrix $C\in \R^{\left((d+1)^2+1\right)\times (d+1)}$ exists.

Algorithm 1 of~\cite{jubran2019fast} shows how to directly construct such a matrix $C$.
However, the running time is $O(n^2 d^2)$ which is undesirable.
To accelerate the running time, Jubran et al.~\cite{jubran2019fast} apply the following idea.
\begin{enumerate}
	\item For each $i\in [n]$, set $p_i\in \R^{(d+1)^2}$ as the concatenation of the $(d+1)^2$ entries of $(a_i, \mathbf{b}_i)\cdot (a_i, \mathbf{b}_i)^\top$.
	Let $P$ be the collection of these points $p_i$.
	Then our objective is reduced to finding a subset $S\subseteq P$ of size $(d+1)^2+1$ such that the convex hull of $S$ contains $\overline{P} = \frac{1}{n}\cdot \sum_{i\in [n]} p_i$.
	\item Compute a balanced partition $P_1,\ldots, P_k$ of $P$ into $k=3(d+1)^2$ clusters of roughly the same size.
	By the Caratheodory's Theorem, there must exist at most $(d+1)^2+1$ partitions $P_i$ such that the convex hull of their union contains $\overline{P}$.
	The main issue is how to these partitions $P_i$ efficiently.
	\item To address this issue, Jubran et al.~\cite{jubran2019fast} compute a sketch for each partition $P_i$ including its size $|P_i|$ and the weighted mean 
	\[
	u_i := \frac{1}{|P_i|}\cdot \sum_{j\in P_i} p_j.
	\]
	The construction of sketches costs $O(nd^2)$ time.
	The key observation is that there exists a set $S$ of at most $(d+1)^2+1$ points $u_i$ such that the convex hull of their union contains $\overline{P}$ by the Caratheodory's Theorem.
	Moreover, the corresponding partitions $P_i$ of these $u_i$ are what we need -- the convex hull of $\bigcup_{i\in [n]: u_i\in S} P_i$ contains $\overline{P}$.
	Note that the construction of $S$ costs $O\left(k^2 \left((d+1)^2\right)^2\right) = O(d^8)$ time.
	Overall, it costs $O(nd^2+d^8)$ time to obtain the collection $\bigcup_{i\in [n]: u_i\in S} P_i$ whose convex hull contains $\overline{P}$.
	\item 
	We repeat the above procedure over $\bigcup_{i\in [n]: u_i\in S} P_i$ until obtaining an accurate coreset of size $(d+1)^2+1$.
	By the value of $k$, we note that 
	\[
	\left|\bigcup_{i\in [n]: u_i\in S} P_i\right| \leq n/2,
	\]
	i.e., we half the size of the input set by an iteration.
	Thus, there are at most $\log (n/d)$ iterations and the overall running time is
	\begin{align*}
	\sum_{i=0}^{\log n} \frac{O(n d^2)}{2^i} + O(d^8) \cdot \log (n/d) 
	= O\left(nd^2+d^8 \log (n/d)\right).
	\end{align*}
\end{enumerate}

\end{document}